\documentclass{article}

\usepackage[preprint]{neurips_2025}
\usepackage{amsmath, amsthm}
\usepackage{mathtools}
\usepackage{graphicx}
\usepackage{wrapfig}
\usepackage{verbatim}

\usepackage[utf8]{inputenc} 
\usepackage[T1]{fontenc}    
\usepackage{hyperref}       
\usepackage{url}            
\usepackage{booktabs}       
\usepackage{amsfonts}       
\usepackage{nicefrac}       
\usepackage{microtype}      
\usepackage{xcolor}         
\usepackage{multirow}

\newtheorem{theorem}{Theorem}[section]
\newtheorem{lemma}{Lemma}

\newtheorem{prop}{Proposition}

\DeclareMathOperator{\sgn}{sgn}

\title{CoFrGeNet: Continued Fraction Architectures for Language Generation}

\author{%
  Amit Dhurandhar\\
  IBM Research \\
  \texttt{adhuran@us.ibm.com} \\
  \And
  Vijil Chenthamarakshan \\
  IBM Research \\
  \texttt{ecvijil@us.ibm.com} \\
  \AND
  Dennis Wei \\
  IBM Research \\
  \texttt{dwei@us.ibm.com} \\
  \And
  Tejaswini Pedapati \\
  IBM Research \\
  \texttt{tejaswinip@us.ibm.com} \\
  \And
  Karthikeyan Natesan Ramamurthy \\
  IBM Research \\
  \texttt{knatesa@us.ibm.com} \\
  \And
  Rahul Nair \\
  IBM Research \\
  \texttt{rahul.nair@ie.ibm.com} \\
}

\begin{document}

\maketitle

\begin{abstract}
Transformers are arguably the preferred architecture for language generation. In this paper, inspired by continued fractions, we introduce a new function class for generative modeling. The architecture family implementing this function class is named CoFrGeNets - Continued Fraction Generative Networks. We design novel architectural components based on this function class that can replace Multi-head Attention and Feed-Forward Networks in Transformer blocks while requiring much fewer parameters. We derive custom gradient formulations to optimize the proposed components more accurately and efficiently than using standard PyTorch-based gradients. Our components are a plug-in replacement requiring little change in training or inference procedures that have already been put in place for Transformer-based models thus making our approach easy to incorporate in large industrial workflows. We experiment on two very different transformer architectures GPT2-xl (1.5B) and Llama3 (3.2B), where the former we pre-train on OpenWebText and GneissWeb, while the latter we pre-train on the docling data mix which consists of nine different datasets. Results show that the performance on downstream classification, Q\& A, reasoning and text understanding tasks of our models is competitive and sometimes even superior to the original models with $\frac{2}{3}$ to $\frac{1}{2}$ the parameters and shorter pre-training time. We believe that future implementations customized to hardware will further bring out the true potential of our architectures.
\end{abstract}
\begin{figure}[htbp]
\vspace{-.3cm}
\centering
\includegraphics[width=.85\textwidth]{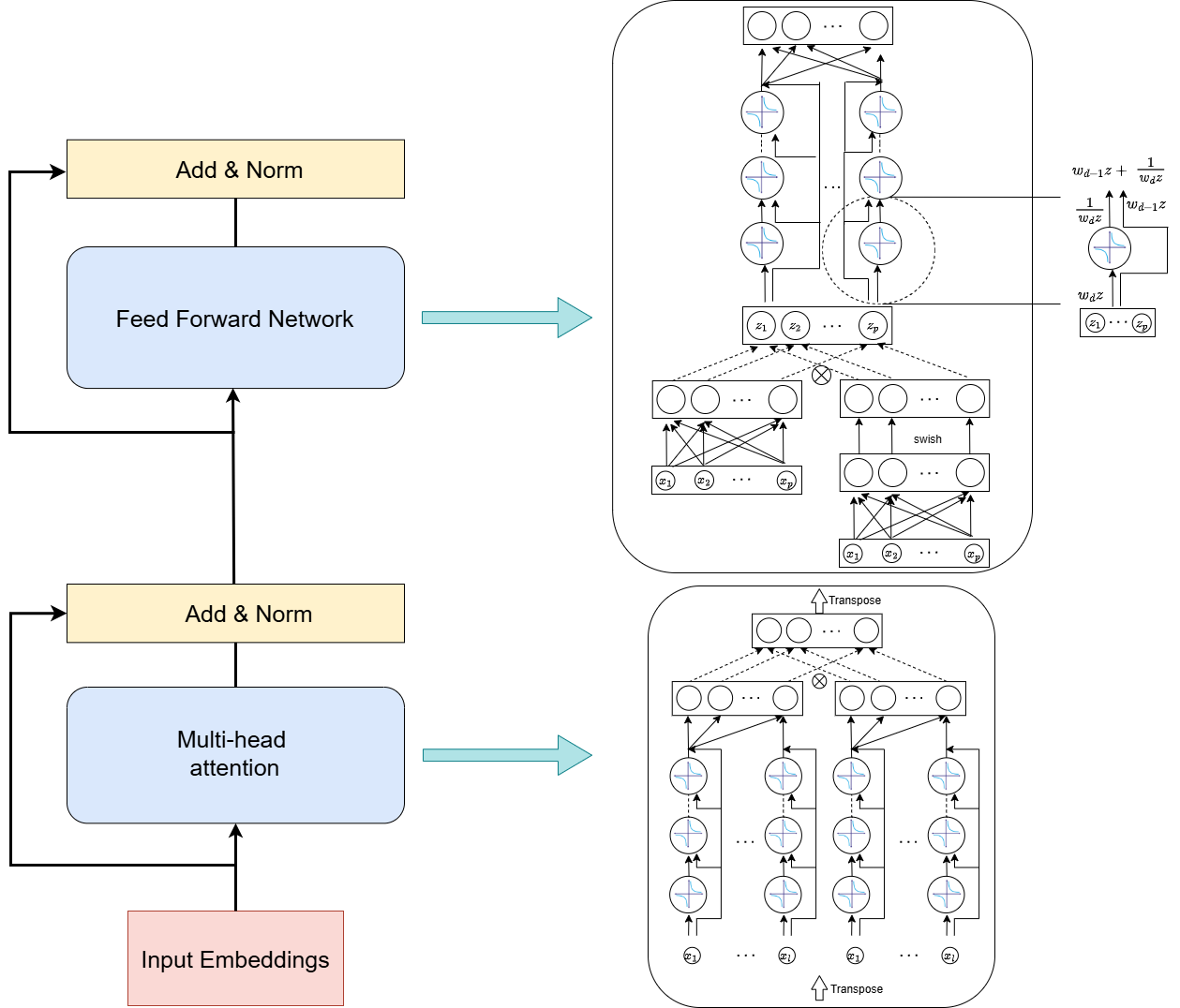}
\caption{Above we see a Transformer block consisting of attention and FFN layers. We propose candidate CoFrNet architectures for Transformer (causal) attention and FFN layers. The circles with the blue curves denote the $\frac{1}{x}$ non-linearity in our architectures. The zoomed out image on the far right shows the mapping between the pictorial representation and the actual equations. 
Details of the architectures are discussed in section \ref{sec:meth}.
}
\label{fig:transformermap}
\vspace{-.25cm}
\end{figure}
\section{Introduction}
Since OpenAI's ChatGPT release at the end of 2022, Large Language Models (LLMs) \citep{radford2019language} have been getting increasingly infused into multiple user applications and platforms across the world. The most prevalent architecture behind these models is the Transformer architecture \citep{attention2017}, which consists of an (multi-head) Attention block and a Feed Forward Network (FFN) with single large hidden layer. In this paper, we propose novel architectural components based on a radically different function class inspired by continued fractions. Taking inspiration from \citep{cofrnets}, where continued fraction architectures \emph{CoFrNets} were introduced for the supervised setting, we build new architectures for the generative setting providing alternatives for attention and FFN in Transformer blocks. 

Given a canonical form for continued fractions $a_0+\frac{1}{a_1+\frac{1}{a_2+\cdots}}$ (ladder like structure) where, $a_k$s are complex numbers, CoFrNets \citep{cofrnets} were introduced for supervised learning problems where in place of the $a_k$s, linear functions of the input $x\in \mathbb{R}^p$ are computed by taking the inner product of $x$ with weight vector $w_k\in \mathbb{R}^p$ in each layer $k$ (or also referred to as step of the ladder).\footnote{A constant term is assumed to be absorbed in $x$.} The reciprocal of the function thus far is applied as a nonlinearity in each layer leading to the following kind of form for a single CoFrNet ladder: 
\begin{align}
w_0x+\frac{1}{w_1x+\frac{1}{w_2x+\cdots}} \label{eqn:cofrnet}
\end{align}
Here $w_k$s are the learnable parameters. Essentially, the input $x$ is passed to each layer which gets multiplied by the corresponding parameter vectors and the reciprocal of the values of the previous layer are added to this. This simple architecture was shown to have universal approximation capabilities when we ensemble enough of these ladders. However, the above contributions were for the supervised setting and it is not clear if such architectures can also be built for representation learning and sequence generation, where we: i) Need to produce multi-dimensional outputs, ii) learn richer functions and iii) model sequences causally i.e. learning parameters that depend only on prior tokens. Moreover, the $\frac{1}{x}$ non-linearity is inefficient to compute in forward and backward passes especially when the depth $d$ and number of ladders $L$ is large. This is because one has to compute the inverse $d\times L$ times and it is known that division is many times slower than multiplication in modern hardware.
We address the above challenges in this paper by making the following contributions that distinguish it significantly from \citep{cofrnets}:

    \textbf{1)} We propose \emph{novel continued fraction architectures} for (causal) attention and FFNs as depicted in Figure \ref{fig:transformermap}. We call our architecture with both components replaced as \textbf{Co}ntinued \textbf{Fr}action \textbf{Ge}nerative \textbf{Net}work (CoFrGeNet). We report results replacing either FFN or attention or both offering the possibility to the user of replacing only one or both of the components for their application. Even replacing one component can offer significant parameter and training time savings as seen in our experiments.    
    \textbf{2)} We propose an \emph{alternative representation} for the ladders and derive custom formulas for the gradients that reduces the number of divisions from $d$ to a constant of just $1$ for a $d$-depth ladder. This greatly enhances both training and inference efficiency.    
    \textbf{3)} We propose a \emph{custom training schedule} to update CoFrGeNet parameters. This is described in section \ref{sec:exp}.  
    \textbf{4)} We pre-train our models on OpenWebText (OWT) \citep{owt}, GneissWeb \citep{gneissweb} and the docling data mix \citep{Docling} showing that our models are \emph{competitive or outperform} the corresponding Transformer models. We compare with Transformers since we are replacing its components making it a fair comparison. For an apples-to-apples comparison with other model architectures such as Mamba \citep{gu2024mamba} one would want to replace its hidden state function with novel (to be designed) CoFrNet components, which would be a significant independent contribution in itself that we leave for future work.

\section{Preliminaries}
\label{sec:prelim}
We introduce notation and also discuss some of properties of continued fractions. The generalized form for a continued fraction is $a_0+\frac{b_1}{a_1+\frac{b_2}{a_2+\cdots}}$, where $a_k$s and $b_k$s can be complex numbers. If none of the $a_k$ or $b_k$ are zero $\forall k\in \mathbb{N}$, then using equivalence transformations \citep{cfbook}, one can create simpler equivalent forms where either the $b_k=1$ or the $a_k=1$ $\forall k\in \mathbb{N}$, with $a_0=0$ in the latter form. A more concise way to write these two forms is as follows: i) $a_0+\frac{1}{a_1+\frac{1}{a_2+\cdots}}\equiv a_0+ \frac{1}{a_1+}\frac{1}{a_2+\cdots}$ and ii) $\frac{b_1}{1+\frac{b_2}{1+\cdots}} \equiv \frac{b_1}{1+} \frac{b_2}{1+\cdots}$. Form i) is known as the \emph{canonical form}.  One of the nice properties of continued fractions is that in representing any real number with natural number parameters $a_k,b_k\in \mathbb{N}$, the rational approximations formed by any of its finite truncations (termed \emph{convergents}) are closer to the true value than any other rational number with the same or smaller denominator. A continued fraction is therefore the best possible rational approximation in this precise sense \citep{cfbook,cfchap}.

In this work, we consider continued fractions in canonical form, with partial numerators $b_k = 1$ for $k = 1, \dots, d$ and depth $d$. We thus view continued fractions as functions $f$ of the partial denominators, where we separate $a_0$ from the others and use $a \coloneqq (a_1, \dots, a_d)$ as a shorthand. Hence we write 
\begin{align}\label{eqn:CFcanon}
f(a_0, a) = a_0 + \frac{1}{a_1 +} \frac{1}{a_2 + } \cdots \frac{1}{a_{d-1} +} \frac{1}{a_d} = a_0 + \tilde{f}(a), 
\end{align}
where we also define $\tilde{f}(a)$ as the ``fractional part'' of $f(a_0, a)$.

Another way of representing a continued fraction is in terms of \emph{continuants}, which we describe next. The continued fraction in \eqref{eqn:CFcanon} can be expressed as the following ratio of polynomials $K_{d+1}$ and $K_{d}$,
\begin{equation}\label{eqn:CFcontRatio}
f(a_0, a) = \frac{K_{d+1}(a_0, \dots, a_d)}{K_d(a_1, \dots, a_d)}.
\end{equation}
Polynomials $K_{d}$, $K_{d+1}$ are part of a sequence of polynomials $K_k$, $k = 0, 1, \dots$, known as \emph{continuants}. They satisfy the recursion
\begin{align}
    &K_0 = 1, \qquad K_1(a_d) = a_d,\label{eqn:initCont}\\
    &K_k(a_{d-k+1}, \dots, a_d) = a_{d-k+1} K_{k-1}(a_{d-k+2}, \dots, a_d) + K_{k-2}(a_{d-k+3}, \dots, a_d).\label{eqn:recurCont}
\end{align}
Using \eqref{eqn:recurCont}, \eqref{eqn:CFcontRatio} can also be written as 
\begin{align}
f(a_0, a) = a_0 + \frac{K_{d-1}(a_2, \dots, a_d)}{K_d(a_1, \dots, a_d)}, \qquad \text{hence } \tilde{f}(a) = \frac{K_{d-1}(a_2, \dots, a_d)}{K_d(a_1, \dots, a_d)}.\label{eqn:CFcontRatio2}
\end{align}

We will exploit the formalism of continuants later for two purposes: first, as a means of computing continued fractions, and second, to derive closed-form expressions for their gradients. This leads to benefits in the forward direction, in terms of speeding up inference, and also in the backward direction, speeding up training, 
compared to standard backpropagation through the multiple layers of a continued fraction. 
While the original CoFrNet work \citep{cofrnets} used this formalism for the limited purpose of local feature-based explanations, here we derive new results making them an integral part in training our architectures.


To construct networks out of continued fractions, we let the partial denominators $a_k$ be affine functions of an input $x$, $a_k = w_k x$, where $w_k$ is a row vector and a $1$ is prepended to the elements of $x$ so that the corresponding coefficient $w_{k0}$ is the intercept or ``bias'' term. We will often refer to a continued fraction with $a_k = w_k x$ as a (CoFrNet) ``ladder'', and we will also construct ensembles of such ladders. 
Throughout the paper we denote the input or embedding dimension by $p$, the number of ladders in an ensemble by $L$, and sequence length by $l$, 
unless specified otherwise.

\section{Related Work}

A brief historical perspective on artificial neural networks is provided in the appendix. Turning our focus to language modeling with neural networks, Recurrent Neural Networks (RNNs), a class of networks with recurrent connections where the output of a neuron at a time step is fed to the input of the neuron at the next time step, were successful in many tasks such as machine translation \citep{sutskever2014sequence} and language modeling \citep{45446}. 
The encoder-decoder Transformer model proposed in \citep{attention2017}, avoids recurrence and relies on attention alone to draw dependencies between the input and output, and these models have  revolutionized language modeling. 
The two early successful transformer architectures that have led to a series of models include the Generative Pre-trained Transformer (GPT) \citep{radford2018improving} and Bidirectional Encoder Representations from
Transformers (BERT) \citep{devlin2019bert}. These pre-trained models can be then \textit{fine-tuned} on relatively small datasets \citep{raffel2020exploring, chung2024scaling, wang2022super} leading to good performance on even unseen tasks. Transformer models, because of their uncompressed view on the entire sequence, show measurable improvement in performance over RNNs, but the attention mechanism scales quadratically with sequence length, as opposed to the linear time generation complexity of RNNs. Given this multiple approximations have been proposed to model attention in Transformers more efficiently. Works such as Synthesizer \citep{syna} and Linformer \citep{linformera} try to make attention linear complexity, while  Mixture-of-depths attention \citep{mixa} and Sliding Window attention \citep{slidea} limit the number of attended tokens in a sequence. Slim attention \citep{slima} does away with the value parameter matrix and models it as a function of the key matrix. Multi-query attention \citep{mqa} and its generalization Grouped Query attention \citep{gqa} limit the number of distinct keys thus reducing parameter count and increasing efficiency. Sparse attention approaches \citep{sparsea} typically attend to local context and sparsely to further away tokens (a.k.a. global context).

Aside from RNNs and Transformers, State-Space Models (SSMs) have also been quite popular. Models such as S4 \citep{gu2022efficiently} and Mamba \citep{gu2024mamba} are recurrent like RNNs, but can handle long range dependencies. The latter selectively propagates information based on the current token making it closer to the modeling power of Transformers, while scaling linearly in sequence length. More recently, Diffusion Models inspired by non-equilibrium statistical physics \citep{pmlr-v37-sohl-dickstein15} have gained traction. The attractive aspect of these models is that generation does not have to be auto-regressive and can happen in parallel. 
In \citep{sahoo2024simple}, the authors propose a simple Masked Diffusion Language Model (MDLM) using an effective training recipe that narrows the gap of diffusion and autoregressive methods in language modeling. Nonetheless, Transformers are still the state-of-the-art in language generation and hence we chose to modify critical components of this architecture.





\begin{figure}[t]
\centering
\includegraphics[width=.42\textwidth]{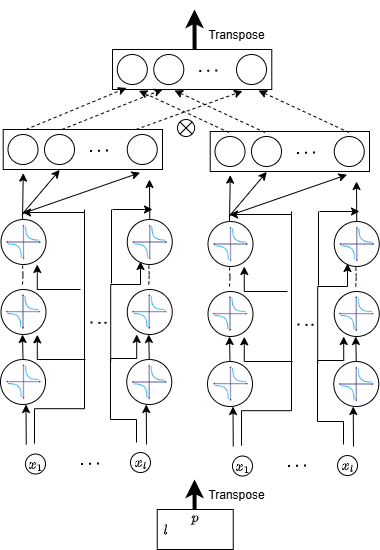}
\hspace{25mm}
\includegraphics[width=.32\textwidth]{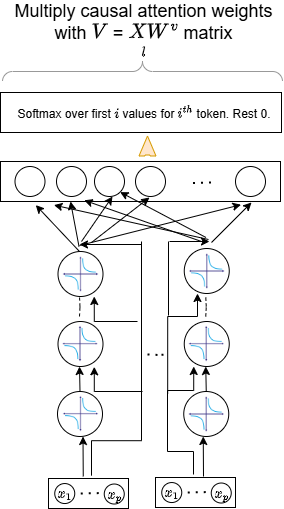}
  \caption{Two CoFrNet architectures to simulate attention a.k.a. causal token-token mixing. For the left architecture (CAttnU) a transpose is taken of the dimension $\times$ sequence length part of the input tensor and the output is transposed back to make it consistent with the later layers. The transpose makes the tokens mix, while upper triangular connections in the second to last layer in the architecture as well as the restricted structure of the ladders make sure information is \emph{only} shared from previous tokens to following tokens and not bi-directionally (a.k.a. causal sharing). It consists of two ensembles of univariate CoFrNet ladders each of which then have an upper triangular linear layer on top. The representations formed are then element wise multiplied to form the final representation. The element wise multiplication produces interaction terms that otherwise would not occur, significantly enhancing representation power without compromising the causal information flow. The right architecture (CAttnM) we do not transpose the input. We use $L$ CoFrNet ladders that get mapped to a sequence length size embedding which corresponds to attention weights for that token. To maintain causality attention weights are computed only over the prior tokens. These then like in standard attention are used to weight the embeddings in the (value) $V$ matrix.
  }
\label{fig:attn}
\vspace{-.25cm}
\end{figure}

\begin{table}[htbp]
\small
\centering
\caption{Scale of parameters for different architectural components. Here $\alpha>>1$ is expansion factor for FFNs in Transformer blocks. The savings in parameters when replacing FFNs can be significantly high as low $d$ and $L$ values are typically sufficient for competitive performance. For attention replacement the savings can be high if $l$ is similar order of magnitude to $p$, which is seen in many architectures (viz. GPT, Llama, etc.).}
\vspace{3mm}
\begin{tabular}{|c|c|c|c|c|}
  \hline
    \textbf{Attention} & \textbf{CAttnU}& \textbf{CAttnM} & \textbf{FFN} & \textbf{Cffn} \\
 \hline\hline
 $4p^2$& $l(2d+l+1)$ & $L(p+l)+p^2$& $2\alpha p^2$& $Lp(d+1)+2p^2$\\
 \hline
  \end{tabular}
 \label{tab:parmcnts}
\end{table} 


\section{Methodology}
\label{sec:meth}

\subsection{Architectures}
We now describe our novel continued fraction architectures that can potentially be used instead of attention and FFN layers in Transformer blocks.

\subsubsection{Replacement for Attention}
In Figure \ref{fig:attn}, we see two potential architectures that perform causal token-token mixing. In the \emph{left architecture}, we take a transpose of the input tensor relative to the embedding dimension and sequence length, which has been done in MLP-Mixer type models \citep{mlpmixer} employed for supervised problems. However, mixing a dimension across tokens arbitrarily will lead to \emph{non-causal} training as the model will get trained assuming access to tokens that follow a given token. 
To handle this we have univariate ladders -- note an input now is a particular dimension across all $l$ tokens -- where, $x_1$ will get different dimensions of the first token in the sequence, $x_2$ will get different dimensions of the second token in the sequence and so on. Hence, $x_1$ can affect all tokens, but $x_2$ can affect all but $x_1$. This is why we have upper triangular linear layer in each ensemble of the architecture. Note that having $p$-variate ladders would break the causal transfer even with upper triangular linear layers as output from each of the ladders would be a function of all tokens. Hence, we have this restricted structure to maintain the causal information constraints else generations are incoherent. We then do element wise multiplication to obtain cross-terms in the variables as the ladders are univariate leading to richer representations. In particular, if depth of the ensembles $d=2$, where  $w^{(1)}_{0}$, $w^{(2)}_{0}$ are parameter vectors at depth 1 and $w^{(1)}_{1}$, $w^{(2)}_{1}$ are parameter vectors at depth 2 for the left and right ensembles respectively, then if $\odot$ implies element-wise multiplication and $\circ -1$ implies element-wise reciprocal we would get:

$y_1 = w^{(1)}_{0}\odot x+{(w^{(1)}_{1}\odot x)}^{\circ -1}$ and $\quad y_2 = w^{(2)}_{0}\odot x+{(w^{(2)}_{1}\odot x)}^{\circ -1}$. 

Let $U_1$ and $U_2$ denote upper triangular parameter matrices then, $O =  U_1y_1 \odot U_2y_2$. $O$  is the $l$ dimensional output produced per input $x$. In our case we will get $p$ such outputs. The tensor containing these $p$ outputs is then transposed back to get a $l\times p$ tensor, which later layers expect.

Now considering the \emph{right architecture} with two ladders (i.e. $L=2$) of depth $2$, a $L\times l$ (full) parameter matrix $F$ and Csoftmax to denote softmax applied causally (i.e. $i^{\text{th}}$ token is a convex combination of the first $i-1$ tokens) with notation from above we have attention weights given by, 

$A = \text{Csoftmax}([y_1, y_2]F)$, where in this case $y_1 = {w^{(1)}_{0}}^{T} x+\left({w^{(1)}_{1}}^{T} x\right)^{-1}$ and $\quad y_2 = {w^{(2)}_{0}}^{T} x+{\left({w^{(2)}_{1}}^{T} x\right)}^{ -1}$ as no transpose of the input tensor is taken and hence $x$, $w$ are $p$ dimensional. If $V=XW^v$ denotes a value matrix like in standard attention where $W^v$ is a $p\times p$ parameter matrix, then the output $O$ is given by: $O = AV$, which would be $l\times p$ tensor.

\begin{wrapfigure}{r}{0.5\textwidth}
\vspace{-1cm}
\centering
\includegraphics[width=.42\textwidth]{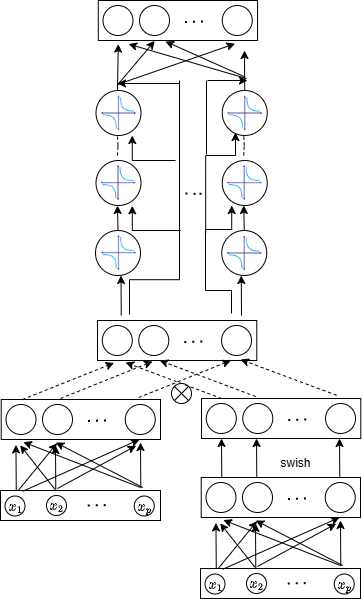}
\caption{CoFrNet architecture simulating FFNs -- Cffn -- in a transformer block (\textbf{more recent than the ICML version}). We create a gated \emph{non-expanded} (i.e. $\alpha=1$) representation that we pass to the CoFrNet ladders. No transpose is taken and hence feature mixing in either direction does not interfere with causal generation which is why we have a linear layer on top. Again the collapsed implementation is described in section \ref{sec:contimpl}.}
\label{fig:mlp}
\vspace{-.8cm}
\end{wrapfigure}
\subsubsection{Replacement for FFNs}
For FFNs we simply require feature mixing so no transpose is taken and all features can mix. Hence, we create ensembles of $p$-variate ladders with a linear layer at the end as seen in Figure \ref{fig:mlp}. 

Note that here one could have an arbitrary number of ladders in the ensemble and one projects to $p$ dimensions using the linear layer. The input to the ladders is a gated non-expanded (i.e. $\alpha=1$) representation. Not performing expansion produces significant parameter savings as seen in the experiments. Expressions depicting the scale of parameters of different architectural components are shown in Table \ref{tab:parmcnts}.


\subsection{Architecture for Continued Fraction Ensembles and Continuant-Based Implementation}
\label{sec:contimpl}

The common element in the architectures in Figures~\ref{fig:attn} and \ref{fig:mlp} is a linear combination of an ensemble of CoFrNet ladders. This subsection describes how we implement these linear combinations of ladders using the continuants introduced in Section~\ref{sec:prelim}.

\noindent\textbf{Architecture} Let us denote by $y \in \mathbb{R}^q$ the output of a linear combination of $L$ ladders, where in general $q$ could be different from the input dimension $p$. We use a superscript $j$ to denote the partial denominators $a^{(j)}_0, \dots, a^{(j)}_d$ corresponding to the $j$th ladder, where $a^{(j)}_k = w^{(j)}_k x$. Then based on \eqref{eqn:CFcanon}, the $i$th output component $y_i$ is given by
\begin{equation}\label{eqn:yi}
    y_i = \sum_{j=1}^L v_{ij} \left( a^{(j)}_0 + \tilde{f}\bigl(a^{(j)}\bigr) \right) = \sum_{j=1}^L v_{ij} w^{(j)}_0 x + \sum_{j=1}^L v_{ij} \tilde{f}\bigl(a^{(j)}\bigr),
\end{equation}
where $v_{ij}$ are the coefficients of the linear combination. Since the composition of two linear functions is also linear, we may simplify the first term on the right-hand side of \eqref{eqn:yi} to yield 
\begin{equation*}
    y_i = u_i x + \sum_{j=1}^L v_{ij} \tilde{f}\bigl(a^{(j)}\bigr),
\end{equation*}
where $u_i = \sum_{j=1}^L v_{ij} w^{(j)}_0$ is the parameter vector of the overall linear function. Let $U$ be the matrix with rows $u_i$, $i=1,\dots,q$, $V$ the matrix with entries $v_{ij}$, and $W^{(j)}$ the matrix with rows $w^{(j)}_k$, $j=1,\dots,d$. We may then express the overall computation from $x$ to $y$ as
\begin{equation}\label{eqn:linComboLadders}
    y = U x + V z, \qquad z_j = \tilde{f}(a^{(j)}), \qquad a^{(j)} = W^{(j)} x, \qquad j = 1,\dots,L.
\end{equation}
Based on \eqref{eqn:linComboLadders}, we implement a linear combination of ladders using the architecture shown in Figure~\ref{fig:Ladder_Ensemble_Architecture}. At the far left is a linear layer parameterized by $U$ that directly connects input $x$ to output $y$. To the right are $L$ ladders, where for each ladder $j$, a linear layer parameterized by $W^{(j)}$ first computes the partial denominators $a^{(j)}$ before the continued fraction is computed by the ``CF'' layer. The continued fraction outputs $z_j$ are fed to a linear layer parameterized by $V$, whose output is added to yield $y$.

\noindent\textbf{Continuant implementation} We use the continuants representation from Section~\ref{sec:prelim} to compute continued fractions in the CF layer. Specifically, continuants $K_0, K_1, \dots, K_d$ are first computed using the recursion in \eqref{eqn:initCont}, \eqref{eqn:recurCont}. The continued fraction output $\tilde{f}(a^{(j)})$ is then given by the ratio of $K_{d-1}$ and $K_d$ in \eqref{eqn:CFcontRatio2}. The following result shows that the \emph{gradient} of $\tilde{f}(a^{(j)})$ is also given by ratios of continuants.

\begin{prop}\label{prop:agrad}
The partial derivatives of continued fraction $\tilde{f}(a)$ defined in \eqref{eqn:CFcanon} are given by 
\begin{equation}\label{eqn:dCF/da}
    \frac{\partial\tilde{f}(a)}{\partial a_k} = (-1)^k \left(\frac{K_{d-k}(a_{k+1},\dots,a_d)}{K_d(a_1,\dots,a_d)}\right)^2, \qquad k = 1,\dots,d.    
\end{equation}
\end{prop}
\begin{proof}
Using equations \eqref{eqn:CFcanon} and \eqref{eqn:CFcontRatio} we get, 
\[
    \frac{\partial\tilde{f}(a)}{\partial a_k} = \frac{\partial}{\partial a_k} \bigl(f(a_0, a) - a_0\bigr) = \frac{\partial}{\partial a_k} \frac{K_{d+1}(a_{0},\dots,a_d)}{K_d(a_1,\dots,a_d)} - 0 
\]
for $k = 1, \dots, d$. We then invoke Lemma~2 stated in the appendix.
\end{proof}
\begin{wrapfigure}{r}{0.45\textwidth}
\vspace{-.5cm}
\centering
\includegraphics[width=0.45\textwidth]{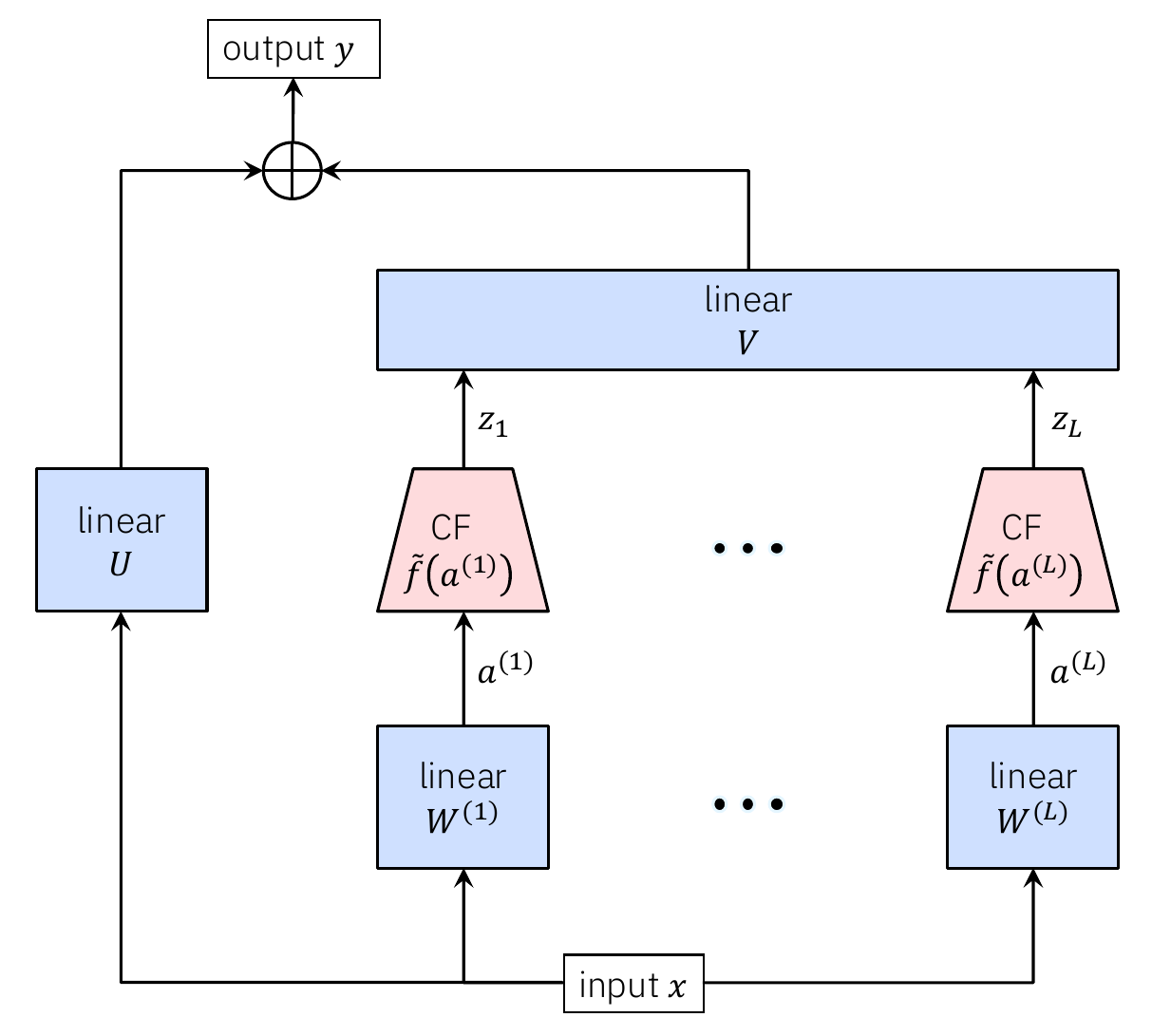}
\caption{Architecture for implementing a linear combination of CoFrNet ladders (CF stands for continued fraction).}
\label{fig:Ladder_Ensemble_Architecture}
\vspace{-.6cm}
\end{wrapfigure}
To take advantage of Proposition~\ref{prop:agrad}, we implement the CF layer in Figure~\ref{fig:Ladder_Ensemble_Architecture} as a custom PyTorch function (\texttt{torch.autograd.Function}). This allows the continuants $K_0, \dots, K_d$, as well as the reciprocal $1/K_d$, to be computed once during the forward pass and saved for the backward pass. Then to compute the gradient, it suffices to multiply $1/K_d$ by other continuants, square the ratios, and change some signs.

\noindent\textbf{Advantages}
Using continuants to compute each continued fraction $\tilde{f}(a^{(j)})$ \eqref{eqn:CFcontRatio2} and its gradient \eqref{eqn:dCF/da} requires only one division, by the same quantity $K_d$. As noted above, the reciprocal $1/K_d$ can be computed once and then reused in all ratios of continuants that are required. As seen from \eqref{eqn:recurCont}, all continuants up to $K_d$ can be computed recursively through $O(d)$ multiplications and additions. This continuants approach yields a major improvement in efficiency over the ``literal'' approach taken in the original CoFrNet work \citep{cofrnets}, which performs one division per layer following the standard representation of a continued fraction \eqref{eqn:cofrnet}. The reduction from $d$ divisions to $1$ is especially significant when ladders are made deep. It applies to both inference and training, since backpropagation through a standard PyTorch implementation of \eqref{eqn:cofrnet} also requires $d$ divisions. 
It is widely known that \emph{divisions are significantly more expensive in current hardware} --- typically an order of magnitude slower --- than multiplications or additions.
%
Moreover, having to divide just once can result in \emph{better numerical stability}.

\noindent\textbf{Avoiding poles and clipping} Equation~\ref{eqn:CFcontRatio2} shows that a continued fraction is equivalent to a rational function, and hence it can suffer from divergence when the denominator $K_d$ goes to zero (these locations are known as \emph{poles} in the context of rational functions). We mitigate this issue using a similar approach as \citep{cofrnets}, namely changing the denominator from $K_d$ to $\sgn(K_d) \max(\lvert K_d \rvert, \epsilon)$ to ensure that it has absolute value at least $\epsilon > 0$. Importantly however, this modification is done only once to $K_d$ as opposed to before every one of the $d$ divisions in \citep{cofrnets}. This may result in less loss of representation power compared to \citep{cofrnets}.

We also maintain the minimum and maximum values that each ladder produces during training. During testing we project or clip predictions to lie in this range so that outputs far away from those seen during training are not produced thus guarding against outlier test predictions.

\section{Experiments}
\label{sec:exp}
\subsection{Setup}
\vspace{-.2cm}
We now perform experiments, where we compare with GPT2-xl (1.5B) first pre-trained on OpenWebText (OWT) \citep{owt} and then on the GneissWeb 35B (GW) \citep{gneissweb} datasets. We compare with three variants of ours i) CoFrGeNet-F, where the FFN is replaced by CoFrNet, ii) CoFrGeNet-A, where the attention is replaced by CoFrNet and iii) CoFrGeNet, where both FFN and attention are replaced. We report results with the CAttnM architecture when attention is replaced as it led to slightly better results than CAttnU in many cases. We also compare with Dense Synthesizer (Synthesizer-D) \citep{syna} which is closest to our CAttnM architecture and an established sparse attention approach (Sparse Attn) \citep{sparsea}. To test the efficacy of CoFrNet on a different architecture we experiment with Llama-3.2B pre-trained on the docling data mix \citep{Docling} of 2T tokens. The data mix contains web (DCLM2, DCLM3Plus \citep{li2024datacomplm}), multilingual (FineWeb-2-edu \citep{fineweb-edu}), code (Starcoder, stack-edu \citep{stack-edu-finemath}), math (Finemath \citep{stack-edu-finemath}, Infiwebmath \citep{infimath}, opc-fineweb-math-corpus \citep{opc}) and synthetic data (Cosmopedia \citep{cosmopedia}), which is heavily used to train models for diverse document understanding. The Llama models already use an efficient form of attention namely Grouped Query Attention (GQA) and hence are a natural efficient attention baseline.

\noindent\textbf{Evaluations:} We report perplexity on Penn Tree Bank (PTB) \citep{Marcus93_PTB}, Wikitext2 \citep{Merity2017_Wikitext}, Wikitext103 \citep{Merity2017_Wikitext}, Lambada \citep{Paperno2016_LAMBADA}, AgNews \citep{AgNews2015} and One Billion Words (LM1B) \citep{Chelba2014_lm1b} datasets. We use a stride of 512 for wikitext2, wikitext103 as recommended in these works. For all the other datasets, we use a stride of 256.
We then fine tune our models on GLUE \citep{glue} (classification) tasks and compare accuracies as done in previous works \citep{diffusion}. We average results over five runs.
\begin{table}[htbp]
    \centering
\scriptsize
\centering
\caption{Downstream task accuracies (best results bolded) on GLUE benchmark after finetuning. The first column is the pre-training dataset. Standard deviations are reported in Table \ref{tab:glue_sd} in the appendix.}
\vspace{1mm}
\begin{tabular}
{|c|c|c|c|c|c|c|c|c|c|}
  \hline
    \textbf{Data} &\textbf{Model} & \textbf{MNLI} & \textbf{QQP}& \textbf{QNLI} & \textbf{SST2} & \textbf{COLA} & \textbf{MRPC} & \textbf{RTE} & \textbf{WNLI}\\
 \hline\hline
 \multirow{6}{*}{OWT} & GPT2-xl \tiny{(1.5B)}& $86.89$&$88.93$ &$91.35$&$93.56$&$81.78$&$79.83$&$60.27$&$58.28$\\
 &CoFrGeNet-F \tiny{(985M)}&$\bf 87.26$&$\bf 89.95$ &$\bf 91.89$&$\bf 94.16$&$\bf 82.59$&$\bf 80.21$&$\bf 61.35$&$\bf 58.30$\\
 &CoFrGeNet-A \tiny{(1.21B)}&$86.94$&$89.31$ &$91.74$&$93.83$&$81.77$&$79.89$&$60.91$&$58.28$\\
 &CoFrGeNet \tiny{(798M)}&$87.11$&$89.36$ &$91.79$&$93.91$&$81.97$&$79.93$&$61.25$&$58.29$\\
&Synthesizer-D \tiny{(1.2B)}&$84.93$&$86.82$ &$90.13$&$91.34$&$80.15$&$77.95$&$59.83$&$58.28$\\
&Sparse Attn \tiny{(1.21B)}&$85.27$&$86.38$ &$90.93$&$92.72$&$80.76$&$77.42$&$59.36$&$58.27$\\
 \hline
 \multirow{6}{*}{GW} & GPT2-xl \tiny{(1.5B)}& $78.28$&$86.83$&$\bf 82.93$&$91.82$&$74.18$&$77.72$&$60.19$&$\bf 58.33$\\
 &CoFrGeNet-F \tiny{(985M)}&$\bf 79.62$&$\bf 87.26$&$82.73$&$\bf 92.36$&$\bf 74.83$&$\bf 78.01$&$\bf 61.35$&$\bf 58.33$\\
 &CoFrGeNet-A \tiny{(1.21B)}&$78.42$&$86.17$&$82.51$&$91.86$&$74.15$&$77.37$&$60.85$&$\bf 58.33$\\
 &CoFrGeNet \tiny{(798M)}&$79.05$&$86.98$&$82.12$&$92.13$&$74.38$&$77.95$&$61.11$&$\bf 58.33$\\
 &Synthesizer-D \tiny{(1.2B)}&$77.56$&$86.35$ &$80.38$&$91.25$&$73.27$&$76.73$&$59.26$&$58.24$\\
&Sparse Attn \tiny{(1.21B)}&$77.67$&$86.41$ &$80.77$&$91.16$&$72.83$&$76.62$&$59.39$&$58.28$\\
 \hline
  \end{tabular}
  \vspace{-1mm}
 \label{tab:glue}
\end{table}
We also compare parameter counts, train time and (per-sample) inference time. We show how the continuants version leads to better train and inference time when compared with the standard implementation of CoFrNets with the improvement mainly attributable to the reduced number of divisions. We provide randomly chosen generations for our variants and GPT2-xl in the appendix. For Llama-3.2B, we evaluate on openbookqa \citep{OpenBookQA2018}, piqa \citep{piqa}, arc-easy \citep{arc}, winogrande \citep{winogrande}, hellaswag \citep{hellaswag}, lambada open AI \citep{radford2018improving}, boolq \citep{boolq} and sciq \citep{SciQ} which cover open domain Q\&A, reasoning and text understanding tasks. We also report the throughput and training time.

\begin{table}[htbp]
\scriptsize
\centering
\caption{Perplexities of the different models with best results bolded.}
\vspace{1mm}
\begin{tabular}{|c|c|c|c|c|c|c|c|}
  \hline
    \textbf{Data} &\textbf{Model} & \textbf{PTB} & \textbf{Wikitxt2}& \textbf{Lbda} & \textbf{AgNews} & \textbf{Lm1b} & \textbf{Wikitxt103}\\
 \hline\hline
 \multirow{6}{*}{OWT} & GPT2-xl \tiny{(1.5B)}&$30.12$ &$18.30$ &$8.66$&$37.13$& $41.20$ & $17.50$\\
&CoFrGeNet-F \tiny{(985M)}&$\bf 29.89$ &$\bf 17.12$& $\bf 8.12$&$\bf 35.72$& $\bf 40.14$ & $\bf 16.14$\\
 &CoFrGeNet-A \tiny{(1.21B)}&$30.02$ &$18.22$& $8.54$&$37.02$& $41.03$ & $17.26$\\
 &CoFrGeNet \tiny{(798M)}&$30.03$ &$17.96$& $8.55$&$36.47$& $40.86$ & $17.17$\\
 &Synthesizer-D \tiny{(1.2B)}&$31.47$ &$19.35$& $9.92$&$39.84$& $41.94$ & $18.91$\\
 &Sparse Attn \tiny{(1.21B)}&$31.23$ &$18.78$& $9.13$&$38.82$& $42.05$ & $18.82$\\
 \hline
 \multirow{6}{*}{GW} & GPT2-xl \tiny{(1.5B)}&$29.07$ &$19.12$&$31.78$ &$45.62$&$52.36$&$18.93$\\
 &CoFrGeNet-F \tiny{(985M)}&$29.72$ &$\bf 18.13$&$\bf 30.52$ &$\bf 41.63$&$\bf 46.83$&$\bf 18.11$\\
 &CoFrGeNet-A \tiny{(1.21B)}&$\bf 28.89$ &$18.77$&$30.98$ &$43.91$&$48.37$&$18.67$\\
 &CoFrGeNet \tiny{(798M)}&$29.08$ &$18.29$&$30.71$ &$42.55$&$48.01$&$18.42$\\
 &Synthesizer-D \tiny{(1.2B)}&$30.83$ &$19.25$& $31.92$&$46.81$& $52.99$ & $19.03$\\
 &Sparse Attn \tiny{(1.21B)}&$29.36$ &$18.95$& $31.23$&$46.38$& $52.83$ & $19.45$\\
 \hline
  \end{tabular}
 \label{tab:ppl}
 \vspace{-0.1cm}
\end{table} 

\noindent\textbf{Parameter Settings:} For pre-training GPT2-xl we use the recommended settings in \url{https://github.com/karpathy/nanoGPT} where, the learning rate is $6\times 10^{-4}$, weight decay is $0.1$, no dropout and maximum iterations is $600K$. For sparse attention (Sparse Attn) we set $g=1$, $w=3$ and $r$ is set to roughly match the number parameters in our CoFrGeNet-A variant for a fair comparison. The values of $g$ and $w$ were set based on experiments conducted in \citep{sparsea} as those produced the best results. For both Synthesizer-D and Sparse Attn we apply a lower triangular mask to the attention weights matrix so as to make the models amenable for auto-regressive generation.

For fine tuning the GPT2-xl model learning rate is $0.25\times 10^{-4}$, batch size is $64$ and no dropout. This is the same for the baselines. For our models the learning rate was $0.125\times 10^{-4}$ with other parameters being the same. These learning rates produced the best results for the respective models.

The Llama variants we pre-train for about 2M iterations. The initial learning rate is $3\times 10^{-4}$ and follows an annealing schedule with no dropout. Adam optimizer is used for both model variants.

For CoFrNets we set $\epsilon = 0.1$. 
We experiment with $d$ equal to $1, 3, 5, 7$ and widths (i.e. number of ladders in an ensemble) also taking the same values when replacing FFNs. We try the same depths and widths when replacing attention.
\begin{wraptable}{r}{.5\textwidth}
\vspace{-.1cm}
\scriptsize
\centering
\caption{Training time and inference time. CoFrGeNet$_B$ is our basic implementation not using continuants. As can be seen using the continuants formalism speeds up training and inference.}
\vspace{1mm}
\begin{tabular}{|c|c|c|c|}
  \hline
    \textbf{Data} &\textbf{Model} & \textbf{Train Time (hrs)}& \textbf{Inf. Time ($\mu$s)} \\
 \hline\hline
 \multirow{5}{*}{OWT} & GPT2-xl &$190$&$643.93$\tiny{$\pm 1.73$}\\
 &CoFrGeNet-F&$186$&$627.48$\tiny{$\pm1.85$}\\
 &CoFrGeNet-A&$186$&$638.26$\tiny{$\pm1.76$}\\
 &CoFrGeNet&$178$&$628.73$\tiny{$\pm1.66$}\\
&CoFrGeNet$_B$&$203$&$5898.72$\tiny{$\pm3.91$}\\
 \hline
 \multirow{5}{*}{GW} & GPT2-xl &$413$&$638.26$\tiny{$\pm 2.73$}\\
 &CoFrGeNet-F&$397$&$627.34$\tiny{$\pm1.65$}\\
 &CoFrGeNet-A&$396$&$625.86$\tiny{$\pm1.78$}\\
 &CoFrGeNet&$387$&$619.78$\tiny{$\pm1.49$}\\
 &CoFrGeNet$_B$&$424$&$5877.87$\tiny{$\pm4.52$}\\
 \hline
  \end{tabular}
 \label{tab:trtime}
 \vspace{-0.4cm}
\end{wraptable} 

\noindent\textbf{Training  Schedule:} 
We employ a dyadic parameter update schedule for our CoFrGeNet components. More specifically, we update only the linear component starting from iteration one, where parameters at higher depths are frozen. Then after half the iterations are done we start updating also the first layer parameters. Then after $\frac{3}{4}^{\text{th}}$ the number of iterations we start updating the depth two parameters and so on. Essentially, depth $i$ parameters are updated for $\frac{t}{2^{i}}$ number of iterations where $t$ is the total number of iterations. We find that this leads to stable training of our architectures as opposed to training all parameters from the start.

\noindent\textbf{Hardware:} We pre-trained the GPT models using $16$ H100 GPUs and distributed data parallel (ddp) training. 
Fine tuning was done using a single A100 GPU for each model. Also inference times were computed for all models using a single A100 GPU. The Llama models were pre-trained using $128$ H100 GPUs with fully sharded distributed data parallel (fsdp) training.

\begin{table}[htbp]
\scriptsize
\centering
\caption{Perplexities of CoFrGeNet (GPT2-xl) variants with (left number) and without (right number) incremental training. As can be seen our training schedule has significant impact. Best results bolded.}
\vspace{1mm}
\begin{tabular}{|p{4mm}|p{21mm}|p{15mm}|p{15mm}|p{15mm}|p{15mm}|p{15mm}|p{15mm}|}
  \hline
    \textbf{Data} &\textbf{Model} & \textbf{PTB} & \textbf{Wikitxt2}& \textbf{Lbda} & \textbf{AgNews} & \textbf{Lm1b} & \textbf{Wikitxt103}\\
 \hline\hline
 \multirow{3}{*}{OWT} 
&CoFrGeNet-F \tiny{(985M)}&$\bf 29.89$, $33.72$ &$\bf 17.12$, $26.71$& $\bf 8.12$, $12.56$&$\bf 35.72$, $42.18$& $\bf 40.14$, $47.28$ & $\bf 16.14$, $22.65$\\
 &CoFrGeNet-A \tiny{(1.21B)}&$30.02$, $38.24$ &$18.22$, $21.82$& $8.54$, $10.92$&$37.02$, $45.52$& $41.03$, $46.21$ & $17.26$, $24.25$\\
 &CoFrGeNet \tiny{(798M)}&$30.03$, $36.77$ &$17.96$, $23.87$& $8.55$, $15.23$&$36.47$, $42.72$& $40.86$, $49.44$ & $17.17$, $23.33$\\
 \hline
 \multirow{3}{*}{GW} 
 &CoFrGeNet-F \tiny{(985M)}&$29.72$, $35.88$ &$\bf 18.13$, $25.55$&$\bf 30.52$, $37.33$ &$\bf 41.63$, $45.46$&$\bf 46.83$, $49.53$&$\bf 18.11$, $20.44$\\
  &CoFrGeNet-A \tiny{(1.21B)}&$\bf 28.89$, $33.71$ &$18.77$, $23.72$&$30.98$, $36.28$ &$43.91$, $45.29$&$48.37$, $52.51$&$18.67$, $21.67$\\
 &CoFrGeNet \tiny{(798M)}&$29.08$, $34.22$ &$18.29$, $22.98$&$30.71$, $36.23$ &$42.55$, $44.39$&$48.01$, $51.91$&$18.42$, $21.67$\\
 \hline
  \end{tabular}
 \label{tab:ppl_ts}
 \vspace{-0.1cm}
\end{table} 

\subsection{Results}
One of the main ways of evaluating if a generative model has learnt good representations is to test it on downstream tasks. In Table \ref{tab:glue} we evaluate how our models perform w.r.t. GPT2-xl on GLUE tasks. 
\begin{wraptable}{r}{.8\textwidth}
\vspace{-.25cm}
\scriptsize
\centering
\caption{Zero-shot accuracies on open domain Q\&A, reasoning and text understanding tasks. The docling data mix of $2$ trillion tokens was used for pre-training.}
\vspace{1mm}
\begin{tabular}{|c|c|c|c|c|c|c|c|c|}
  \hline
    \textbf{Model} & \textbf{openqa} & \textbf{piqa}& \textbf{arc} & \textbf{wino} & \textbf{hswag} & \textbf{lambada} & \textbf{boolq} & \textbf{sciq}\\
 \hline\hline
  Llama \tiny{(3.2B)}& $.282$&$.76$ &$.77$&$\bf .654$&$\bf .503$&$.581$&$\bf .691$&$.941$\\
 CoFrGeNet-F \tiny{(2.1B)}&$.294$&$\bf .764$ &$\bf .778$&$.649$&$.491$&$\bf .583$&$.668$&$\bf .944$\\
 CoFrGeNet-A \tiny{(2.5B)}&$\bf .304$&$.752$ &$.757$&$.646$&$.463$&$.575$&$.633$&$.914$\\
 CoFrGeNet \tiny{(1.8B)}&$.283$&$.751$ &$.751$&$.64$&$.464$&$.571$&$.633$&$.907$\\
Mamba-2 \tiny{(3.2B)}&$\bf .324$ &	$.761$	& $.768$ &	$.615$	& $.486$	& $.548$	& $.655$ &	$.919$
\\
\hline
   \end{tabular}
  \vspace{-5mm}
 \label{tab:docling}
 \vspace{5mm}
\end{wraptable}
We observe that our models are much smaller -- sizes are mentioned next to the names in column two -- yet are better in performance in most cases to the original GPT2-xl model. In fact, they are also better than the linear attention and sparse attention baselines being similar or smaller size. For the Sparse Attn baseline the size reflects the sparsity level or the number of non-zeros. CoFrGeNet-F seems to have the best performance amongst all the variants in most cases. In Table \ref{tab:ppl}, we evaluate how confident the model is in its generations. 
We see in Table \ref{tab:ppl} that again our models are better than GPT2-xl and the efficient attention baselines. Here again CoFrGeNet-F seems to have the best perplexity in most cases consistent with the fine tuning performance.
\begin{wraptable}{r}{.45\textwidth}
\vspace{-.65cm}
\scriptsize
\centering
\caption{Throughput for Llama-3.2B and our variants.}
\vspace{1mm}
\begin{tabular}{|c|c|c|}
  \hline
    \textbf{Model} & \textbf{Tokens/day} & \textbf{Train Time (days)}\\
 \hline\hline
 Llama \tiny{(3.2B)}& $235$B&$8.5$\\
 CoFrGeNet-F \tiny{(2.1B)}&$303$B&$6.6$\\
  CoFrGeNet-A \tiny{(2.5B)}&$250$B&$8$\\
  CoFrGeNet \tiny{(1.8B)}&$315$B&$6.4$ \\
\hline
   \end{tabular}
  \vspace{-1mm}
 \label{tab:doclingspeed}
\end{wraptable}
In Table \ref{tab:trtime}, we compare training and inference times of our models and GPT2-xl. Here we add an additional model CoFrGeNet$_B$ which is the same architecture as CoFrGeNet, but implemented as multi-layer ladders as done in \citep{cofrnets}, without exploiting the continuants formalism. This means a division operation has to be done at every layer of the ladder while training and inferring. As can be seen the training for the continuants version is faster, with inference being almost an order of magnitude faster. In Table \ref{tab:ppl_ts}, we compare the perplexities of our trained models with and without our custom training schedule. As can be seen our training schedule leads to much better performing models as it stabilizes training.

In Table \ref{tab:docling}, we observe similar qualitative behavior for the Llama models even when tested on diverse tasks ranging from open domain Q\&A to reasoning, where CoFrGeNet-F is the best on majority of these tasks, while the other variants are still competitive with the original Llama model. The throughputs are observed in Table \ref{tab:doclingspeed}. We see that our variants are faster than the original Llama where, CoFrGeNet-F and CoFrGeNet take as much as a couple of days less to train. 

These results suggest that across model architectures and tasks our architectural modifications lead to competitive models that are parameter efficient.

\section{Discussion}
\label{sec:disc}
We have proposed novel continued fraction inspired architectures as replacements for attention and FFNs in transformer blocks. This new interesting function class can learn accurate, compact models that are also efficient to train and infer. Our continuant based gradient derivation and implementation facilitated these benefits over and above optimizing these architectures by backpropagating through the layers using standard Pytorch functionalities as done previously \citep{cofrnets}. The custom training schedule for CoFrGeNet specific parameters further helped stabilize and improve performance. In the future, it would be interesting to experiment with other open architectures such as Mamba as well as Mixture-Of-Experts kind of architectures. Inventing new and better CoFrNet architectures for attention and FFNs beyond those proposed in this work is another interesting direction. Also building custom Triton Kernels \citep{tillet2019triton} for our components to further speedup training and inference might be a worthwhile future effort.


As such we believe we have laid the groundwork for continued fraction inspired generative architectures. This could lead to small, efficient to train and accurate generative models across applications and industries. In a way this could further democratize AI as entities with fewer resources could also pre-train good quality models.

Our use of divisions as the non-linearity distinguishes our architecture from others. However, divisions are more expensive than matrix multiplications and additions in modern digital hardware and this is a drawback. We are working towards hardware as well as software solutions for this. For instance, outsourcing divisions to FPGA hardware is something we are seeing initial promise in. Also writing divisions in some other form may be interesting to explore from a software side. Also, there are no implicit safety guards for these models similar to other architectures and so they are susceptible to hallucinations, adversarial attacks and the likes. We hope future research exploiting the specific functional form can implicitly address some of these challenges, which we believe could be very exciting.

\section*{Acknowledgements}
We would like to thank David Cox for pointing us towards the docling data mix as well as the GW dataset and making sure we have resources to run the experiments. We would also like to thank Hajar Gohari's team and Ahmed Nassar for details on these datasets as well as corresponding models and training frameworks. Also special thanks to Kush Varshney, Sriram Raghavan and Ruchir Puri for supporting this effort.

\bibliography{ref}

@software{cosmopedia,
  author = {Ben Allal, Loubna and Lozhkov, Anton and Penedo, Guilherme and Wolf, Thomas and von Werra, Leandro},
  title = {Cosmopedia},
  month = February,
  year = 2024,
  url = {https://huggingface.co/datasets/HuggingFaceTB/cosmopedia}
}

@inproceedings{opc,
  title = {OpenCoder: The Open Cookbook for Top-Tier Code Large Language Models},
  author = {Siming Huang and Tianhao Cheng and Jason Klein Liu and Jiaran Hao and Liuyihan Song and Yang Xu and J. Yang and J. H. Liu and Chenchen Zhang and Linzheng Chai and Ruifeng Yuan and Zhaoxiang Zhang and Jie Fu and Qian Liu and Ge Zhang and Zili Wang and Yuan Qi and Yinghui Xu and Wei Chu},
  year = {2024},
  url = {https://arxiv.org/pdf/2411.04905}
}

@misc{infimath,
      title={InfiMM-WebMath-40B: Advancing Multimodal Pre-Training for Enhanced Mathematical Reasoning}, 
      author={Xiaotian Han and Yiren Jian and Xuefeng Hu and Haogeng Liu and Yiqi Wang and Qihang Fan and Yuang Ai and Huaibo Huang and Ran He and Zhenheng Yang and Quanzeng You},
      year={2024},
      eprint={2409.12568},
      archivePrefix={arXiv},
      primaryClass={cs.CV},
      url={https://arxiv.org/abs/2409.12568}, 
}

@misc{stack-edu-finemath,
      title={SmolLM2: When Smol Goes Big -- Data-Centric Training of a Small Language Model}, 
      author={Loubna Ben Allal and Anton Lozhkov and Elie Bakouch and Gabriel Martín Blázquez and Guilherme Penedo and Lewis Tunstall and Andrés Marafioti and Hynek Kydlíček and Agustín Piqueres Lajarín and Vaibhav Srivastav and Joshua Lochner and Caleb Fahlgren and Xuan-Son Nguyen and Clémentine Fourrier and Ben Burtenshaw and Hugo Larcher and Haojun Zhao and Cyril Zakka and Mathieu Morlon and Colin Raffel and Leandro von Werra and Thomas Wolf},
      year={2025},
      eprint={2502.02737},
      archivePrefix={arXiv},
      primaryClass={cs.CL},
      url={https://arxiv.org/abs/2502.02737}
}

@software{fineweb-edu,
  author = {Lozhkov, Anton and Ben Allal, Loubna and von Werra, Leandro and Wolf, Thomas},
  title = {FineWeb-Edu},
  month = May,
  year = 2024,
  url = {https://huggingface.co/datasets/HuggingFaceFW/fineweb-edu}
}

@article{li2024datacomplm,
      title={DataComp-LM: In search of the next generation of training sets for language models},
      author={Jeffrey Li and Alex Fang and Georgios Smyrnis and Maor Ivgi and Matt Jordan and Samir Gadre and Hritik Bansal and Etash Guha and Sedrick Keh and Kushal Arora and Saurabh Garg and Rui Xin and Niklas Muennighoff and Reinhard Heckel and Jean Mercat and Mayee Chen and Suchin Gururangan and Mitchell Wortsman and Alon Albalak and Yonatan Bitton and Marianna Nezhurina and Amro Abbas and Cheng-Yu Hsieh and Dhruba Ghosh and Josh Gardner and Maciej Kilian and Hanlin Zhang and Rulin Shao and Sarah Pratt and Sunny Sanyal and Gabriel Ilharco and Giannis Daras and Kalyani Marathe and Aaron Gokaslan and Jieyu Zhang and Khyathi Chandu and Thao Nguyen and Igor Vasiljevic and Sham Kakade and Shuran Song and Sujay Sanghavi and Fartash Faghri and Sewoong Oh and Luke Zettlemoyer and Kyle Lo and Alaaeldin El-Nouby and Hadi Pouransari and Alexander Toshev and Stephanie Wang and Dirk Groeneveld and Luca Soldaini and Pang Wei Koh and Jenia Jitsev and Thomas Kollar and Alexandros G. Dimakis and Yair Carmon and Achal Dave and Ludwig Schmidt and Vaishaal Shankar},
      year={2024},
      journal={arXiv preprint arXiv:2406.11794}
}

@inproceedings{SciQ,
    title={Crowdsourcing Multiple Choice Science Questions},
    author={Johannes Welbl and Nelson F. Liu and Matt Gardner},
    year={2017},
    journal={arXiv:1707.06209v1}
}

@inproceedings{boolq,
  title =     {BoolQ: Exploring the Surprising Difficulty of Natural Yes/No Questions},
  author =    {Clark Christopher and Lee Kenton and Chang Ming-Wei and Kwiatkowski Tom and Collins Michael and Toutanova Kristina},
  booktitle = {NAACL},
  year =      {2019}
}

@inproceedings{hellaswag,
    title={HellaSwag: Can a Machine Really Finish Your Sentence?},
    author={Zellers, Rowan and Holtzman, Ari and Bisk, Yonatan and Farhadi, Ali and Choi, Yejin},
    booktitle ={Proceedings of the 57th Annual Meeting of the Association for Computational Linguistics},
    year={2019}
}

@InProceedings{winogrande,
title = {WinoGrande: An Adversarial Winograd Schema Challenge at Scale},
authors={Keisuke, Sakaguchi and Ronan, Le Bras and Chandra, Bhagavatula and Yejin, Choi
},
year={2019}
}

@article{arc,
      author    = {Peter Clark  and Isaac Cowhey and Oren Etzioni and Tushar Khot and
                    Ashish Sabharwal and Carissa Schoenick and Oyvind Tafjord},
      title     = {Think you have Solved Question Answering? Try ARC, the AI2 Reasoning Challenge},
      journal   = {arXiv:1803.05457v1},
      year      = {2018},
}

@inproceedings{OpenBookQA2018,
 title={Can a Suit of Armor Conduct Electricity? A New Dataset for Open Book Question Answering},
 author={Todor Mihaylov and Peter Clark and Tushar Khot and Ashish Sabharwal},
 booktitle={EMNLP},
 year={2018}
}

@inproceedings{piqa,
  author = {Yonatan Bisk and Rowan Zellers and
            Ronan Le Bras and Jianfeng Gao
            and Yejin Choi},
  title = {PIQA: Reasoning about Physical Commonsense in
           Natural Language},
  booktitle = {Thirty-Fourth AAAI Conference on
               Artificial Intelligence},
  year = {2020},
}

@techreport{Docling,
  author = {Deep Search Team},
  month = {8},
  title = {Docling Technical Report},
  url = {https://arxiv.org/abs/2408.09869},
  eprint = {2408.09869},
  doi = {10.48550/arXiv.2408.09869},
  version = {1.0.0},
  year = {2024}
}

@inproceedings{sparsea,
author = {Zaheer, Manzil and Guruganesh, Guru and Dubey, Avinava and Ainslie, Joshua and Alberti, Chris and Ontanon, Santiago and Pham, Philip and Ravula, Anirudh and Wang, Qifan and Yang, Li and Ahmed, Amr},
title = {Big bird: transformers for longer sequences},
year = {2024},
articleno = {1450},
numpages = {15},
location = {Vancouver, BC, Canada},
series = {NeurIPS '24}
}

@article{gqa,
  title={GQA: Training Generalized Multi-Query Transformer Models from Multi-Head Checkpoints},
  author={Ainslie Joshua and Lee-Thorp, James and de Jong, Michiel and Zemlyanskiy, Yury and Lebr{\'o}n, Federico and Sanghai, Sumit},
  journal={Empirical Method in Natural Language Prcessing},
  year={2023}
}

@misc{mqa,
      title={Fast Transformer Decoding: One Write-Head is All You Need}, 
      author={Noam Shazeer},
      year={2019},
      eprint={1911.02150},
      archivePrefix={arXiv},
      primaryClass={cs.NE},
      url={https://arxiv.org/abs/1911.02150}, 
}

@misc{slima,
      title={Slim attention: cut your context memory in half without loss -- K-cache is all you need for MHA}, 
      author={Nils Graef and Andrew Wasielewski},
      year={2025},
      eprint={2503.05840},
      archivePrefix={arXiv},
      primaryClass={cs.LG},
      url={https://arxiv.org/abs/2503.05840}, 
}

@misc{slidea,
      title={Sliding Window Attention Training for Efficient Large Language Models}, 
      author={Zichuan Fu and Wentao Song and Yejing Wang and Xian Wu and Yefeng Zheng and Yingying Zhang and Derong Xu and Xuetao Wei and Tong Xu and Xiangyu Zhao},
      year={2025},
      eprint={2502.18845},
      archivePrefix={arXiv},
      primaryClass={cs.CL},
      url={https://arxiv.org/abs/2502.18845}, 
}

@misc{mixa,
      title={Attention Is All You Need For Mixture-of-Depths Routing}, 
      author={Advait Gadhikar and Souptik Kumar Majumdar and Niclas Popp and Piyapat Saranrittichai and Martin Rapp and Lukas Schott},
      year={2024},
      eprint={2412.20875},
      archivePrefix={arXiv},
      primaryClass={cs.CV},
      url={https://arxiv.org/abs/2412.20875}, 
}

@inproceedings{syna,
author = {Yi Tay and Dara Bahri and Donald Metzler and Da-Cheng Juan and Zhe Zhao and Che Zheng},
title = {Synthesizer: Rethinking Self-Attention in Transformer Models},
year = {2021},
booktitle = {Intl. Conference on Machine Learning}
}

@article{linformera,
  author       = {Sinong Wang and
                  Belinda Z. Li and
                  Madian Khabsa and
                  Han Fang and
                  Hao Ma},
  title        = {Linformer: Self-Attention with Linear Complexity},
  journal      = {CoRR},
  volume       = {abs/2006.04768},
  year         = {2020},
  url          = {https://arxiv.org/abs/2006.04768},
  eprinttype    = {arXiv}
}

@inproceedings{AgNews2015,
author = {Zhang, Xiang and Zhao, Junbo and LeCun, Yann},
title = {Character-Level Convolutional Networks for Text Classification},
year = {2015},
publisher = {MIT Press},
address = {Cambridge, MA, USA},
booktitle = {Proceedings of the 28th International Conference on Neural Information Processing Systems - Volume 1},
pages = {649–657},
numpages = {9},
location = {Montreal, Canada},
series = {NIPS'15}
}

@inproceedings{mlpmixer,
      title={MLP-Mixer: An all-MLP Architecture for Vision}, 
      author={Ilya Tolstikhin and Neil Houlsby and Alexander Kolesnikov and Lucas Beyer and Xiaohua Zhai and Thomas Unterthiner and Jessica Yung and Andreas Steiner and Daniel Keysers and Jakob Uszkoreit and Mario Lucic and Alexey Dosovitskiy},
      booktitle={Computer Vision and Pattern Recognition},
  year={2021}
}

@inproceedings{tillet2019triton,
  title={Triton: an intermediate language and compiler for tiled neural network computations},
  author={Tillet, Philippe and Kung, Hsiang-Tsung and Cox, David},
  booktitle={Proceedings of the 3rd ACM SIGPLAN International Workshop on Machine Learning and Programming Languages},
  pages={10--19},
  year={2019}
}

@inproceedings{glue,
  title={GLUE: A Multi-Task Benchmark and Analysis Platform for Natural Language Understanding},
  author={Alex Wang and Amanpreet Singh and Julian Michael and Felix Hill and Omer Levy and Samuel R. Bowman},
  booktitle={Proceedings of the 24th International Conference on Learning Representations},
  year={2019}
}

@misc{45446,title	= {Exploring the limits of language modeling},author	= {Rafal Jozefowicz and Oriol Vinyals and Mike Schuster and Noam Shazeer and Yonghui Wu},year	= {2016},URL	= {https://arxiv.org/pdf/1602.02410.pdf}}

@inproceedings{Chelba2014_lm1b,
  author       = {Ciprian Chelba and
                  Tom{\'{a}}s Mikolov and
                  Mike Schuster and
                  Qi Ge and
                  Thorsten Brants and
                  Phillipp Koehn and
                  Tony Robinson},
  editor       = {Haizhou Li and
                  Helen M. Meng and
                  Bin Ma and
                  Engsiong Chng and
                  Lei Xie},
  title        = {One billion word benchmark for measuring progress in statistical language
                  modeling},
  booktitle    = {15th Annual Conference of the International Speech Communication Association,
                  {INTERSPEECH} 2014, Singapore, September 14-18, 2014},
  pages        = {2635--2639},
  publisher    = {{ISCA}},
  year         = {2014},
  url          = {https://doi.org/10.21437/Interspeech.2014-564},
  doi          = {10.21437/INTERSPEECH.2014-564},
  bibsource    = {dblp computer science bibliography, https://dblp.org}
}

@article{Marcus93_PTB,
  author       = {Mitchell P. Marcus and
                  Beatrice Santorini and
                  Mary Ann Marcinkiewicz},
  title        = {Building a Large Annotated Corpus of English: The Penn Treebank},
  journal      = {Comput. Linguistics},
  volume       = {19},
  number       = {2},
  pages        = {313--330},
  year         = {1993},
  bibsource    = {dblp computer science bibliography, https://dblp.org}
}

@inproceedings{Paperno2016_LAMBADA,
  author       = {Denis Paperno and
                  Germ{\'{a}}n Kruszewski and
                  Angeliki Lazaridou and
                  Quan Ngoc Pham and
                  Raffaella Bernardi and
                  Sandro Pezzelle and
                  Marco Baroni and
                  Gemma Boleda and
                  Raquel Fern{\'{a}}ndez},
  title        = {The {LAMBADA} dataset: Word prediction requiring a broad discourse
                  context},
  booktitle    = {Proceedings of the 54th Annual Meeting of the Association for Computational
                  Linguistics, {ACL} 2016, August 7-12, 2016, Berlin, Germany, Volume
                  1: Long Papers},
  publisher    = {The Association for Computer Linguistics},
  year         = {2016},
  url          = {https://doi.org/10.18653/v1/p16-1144},
  doi          = {10.18653/V1/P16-1144},
  bibsource    = {dblp computer science bibliography, https://dblp.org}
}

@inproceedings{Merity2017_Wikitext,
  author       = {Stephen Merity and
                  Caiming Xiong and
                  James Bradbury and
                  Richard Socher},
  title        = {Pointer Sentinel Mixture Models},
  booktitle    = {5th International Conference on Learning Representations, {ICLR} 2017,
                  Toulon, France, April 24-26, 2017, Conference Track Proceedings},
  publisher    = {OpenReview.net},
  year         = {2017},
  url          = {https://openreview.net/forum?id=Byj72udxe},
  bibsource    = {dblp computer science bibliography, https://dblp.org}
}

@inproceedings{wang2022super,
  title={Super-NaturalInstructions: Generalization via Declarative Instructions on 1600+ NLP Tasks},
  author={Wang, Yizhong and Mishra, Swaroop and Alipoormolabashi, Pegah and Kordi, Yeganeh and Mirzaei, Amirreza and Naik, Atharva and Ashok, Arjun and Dhanasekaran, Arut Selvan and Arunkumar, Anjana and Stap, David and others},
  booktitle={Proceedings of the 2022 Conference on Empirical Methods in Natural Language Processing},
  pages={5085--5109},
  year={2022}
}

@inproceedings{
sahoo2024simple,
title={Simple and Effective Masked Diffusion Language Models},
author={Subham Sekhar Sahoo and Marianne Arriola and Aaron Gokaslan and Edgar Mariano Marroquin and Alexander M Rush and Yair Schiff and Justin T Chiu and Volodymyr Kuleshov},
booktitle={The Thirty-eighth Annual Conference on Neural Information Processing Systems},
year={2024},
url={https://openreview.net/forum?id=L4uaAR4ArM}
}

@InProceedings{pmlr-v37-sohl-dickstein15,
  title = 	 {Deep Unsupervised Learning using Nonequilibrium Thermodynamics},
  author = 	 {Sohl-Dickstein, Jascha and Weiss, Eric and Maheswaranathan, Niru and Ganguli, Surya},
  booktitle = 	 {Proceedings of the 32nd International Conference on Machine Learning},
  pages = 	 {2256--2265},
  year = 	 {2015},
  editor = 	 {Bach, Francis and Blei, David},
  volume = 	 {37},
  series = 	 {Proceedings of Machine Learning Research},
  address = 	 {Lille, France},
  month = 	 {07--09 Jul},
  publisher =    {PMLR},
  url = 	 {https://proceedings.mlr.press/v37/sohl-dickstein15.html},
}

@inproceedings{
gu2022efficiently,
title={Efficiently Modeling Long Sequences with Structured State Spaces},
author={Albert Gu and Karan Goel and Christopher Re},
booktitle={International Conference on Learning Representations},
year={2022},
url={https://openreview.net/forum?id=uYLFoz1vlAC}
}

@misc{gneissweb,
      title={GneissWeb: Preparing High Quality Data for LLMs at Scale}, 
      author={Hajar Emami Gohari and Swanand Ravindra Kadhe and Syed Yousaf Shah. Constantin Adam and Abdulhamid Adebayo and Praneet Adusumilli and Farhan Ahmed and Nathalie Baracaldo Angel and Santosh Borse and Yuan-Chi Chang and Xuan-Hong Dang and Nirmit Desai and Ravital Eres and Ran Iwamoto and Alexei Karve and Yan Koyfman and Wei-Han Lee and Changchang Liu and Boris Lublinsky and Takuyo Ohko and Pablo Pesce and Maroun Touma and Shiqiang Wang and Shalisha Witherspoon and Herbert Woisetschlager and David Wood and Kun-Lung Wu and Issei Yoshida and Syed Zawad and Petros Zerfos and Yi Zhou and Bishwaranjan Bhattacharjee},
      year={2025},
      eprint={2502.14907},
      archivePrefix={arXiv},
      primaryClass={cs.CL},
      url={https://arxiv.org/abs/2502.14907}, 
}

@misc{owt,
    title={OpenWebText Corpus},
    author={Gokaslan, Aaron and Cohen, Vanya and Pavlick, Ellie and Tellex, Stefanie},
    howpublished={\url{http://Skylion007.github.io/OpenWebTextCorpus}},
    year={2019}
}

@article{radford2019language,
  title={Language models are unsupervised multitask learners},
  author={Radford, Alec and Wu, Jeffrey and Child, Rewon and Luan, David and Amodei, Dario and Sutskever, Ilya and others},
  journal={OpenAI blog},
  volume={1},
  number={8},
  pages={9},
  year={2019}
}

@inproceedings{
gu2024mamba,
title={Mamba: Linear-Time Sequence Modeling with Selective State Spaces},
author={Albert Gu and Tri Dao},
booktitle={First Conference on Language Modeling},
year={2024},
url={https://openreview.net/forum?id=tEYskw1VY2}
}

@article{mcculloch1943logical,
  title={A logical calculus of the ideas immanent in nervous activity},
  author={McCulloch, Warren S and Pitts, Walter},
  journal={The bulletin of mathematical biophysics},
  volume={5},
  pages={115--133},
  year={1943},
  publisher={Springer}
}

@article{rosenblatt1958perceptron,
  title={The perceptron: a probabilistic model for information storage and organization in the brain.},
  author={Rosenblatt, Frank},
  journal={Psychological review},
  volume={65},
  number={6},
  pages={386},
  year={1958},
  publisher={American Psychological Association}
}

@article{rumelhart1986learning,
  title={Learning representations by back-propagating errors},
  author={Rumelhart, David E and Hinton, Geoffrey E and Williams, Ronald J},
  journal={nature},
  volume={323},
  number={6088},
  pages={533--536},
  year={1986},
  publisher={Nature Publishing Group UK London}
}

@article{linnainmaa1976taylor,
  title={Taylor expansion of the accumulated rounding error},
  author={Linnainmaa, Seppo},
  journal={BIT Numerical Mathematics},
  volume={16},
  number={2},
  pages={146--160},
  year={1976},
  publisher={Springer}
}

@article{ivakhnenko1971polynomial,
  title={Polynomial theory of complex systems},
  author={Ivakhnenko, Alexey Grigorevich},
  journal={IEEE transactions on Systems, Man, and Cybernetics},
  number={4},
  pages={364--378},
  year={1971},
  publisher={IEEE}
}

@article{sutskever2014sequence,
  title={Sequence to sequence learning with neural networks},
  author={Sutskever, Ilya and Vinyals, Oriol and Le, Quoc V},
  journal={Advances in neural information processing systems},
  volume={27},
  year={2014}
}

@article{chung2024scaling,
  title={Scaling instruction-finetuned language models},
  author={Chung, Hyung Won and Hou, Le and Longpre, Shayne and Zoph, Barret and Tay, Yi and Fedus, William and Li, Yunxuan and Wang, Xuezhi and Dehghani, Mostafa and Brahma, Siddhartha and others},
  journal={Journal of Machine Learning Research},
  volume={25},
  number={70},
  pages={1--53},
  year={2024}
}

@article{raffel2020exploring,
  title={Exploring the limits of transfer learning with a unified text-to-text transformer},
  author={Raffel, Colin and Shazeer, Noam and Roberts, Adam and Lee, Katherine and Narang, Sharan and Matena, Michael and Zhou, Yanqi and Li, Wei and Liu, Peter J},
  journal={Journal of machine learning research},
  volume={21},
  number={140},
  pages={1--67},
  year={2020}
}

@inproceedings{devlin2019bert,
  title={Bert: Pre-training of deep bidirectional transformers for language understanding},
  author={Devlin, Jacob and Chang, Ming-Wei and Lee, Kenton and Toutanova, Kristina},
  booktitle={Proceedings of the 2019 conference of the North American chapter of the association for computational linguistics: human language technologies, volume 1 (long and short papers)},
  pages={4171--4186},
  year={2019}
}

@article{radford2018improving,
  title={Improving language understanding by generative pre-training},
  author={Radford, Alec and Narasimhan, Karthik and Salimans, Tim and Sutskever, Ilya},
  year={2018},
  publisher={San Francisco, CA, USA}
}

@inproceedings{attention2017,
 author = {Vaswani, Ashish and Shazeer, Noam and Parmar, Niki and Uszkoreit, Jakob and Jones, Llion and Gomez, Aidan N and Kaiser, \L ukasz and Polosukhin, Illia},
 booktitle = {Advances in Neural Information Processing Systems},
 title = {Attention is All you Need},
 year = {2017}
}

@inproceedings{cofrnets,
 author = {Puri, Isha and Dhurandhar, Amit and Pedapati, Tejaswini and Shanmugam, Karthikeyan and Wei, Dennis and Varshney, Kush R},
 booktitle = {Advances in Neural Information Processing Systems},
 editor = {M. Ranzato and A. Beygelzimer and Y. Dauphin and P.S. Liang and J. Wortman Vaughan},
 pages = {21668--21680},
 publisher = {Curran Associates, Inc.},
 title = {CoFrNets: Interpretable Neural Architecture Inspired by Continued Fractions},
 url = {https://proceedings.neurips.cc/paper_files/paper/2021/file/b538f279cb2ca36268b23f557a831508-Paper.pdf},
 volume = {34},
 year = {2021}
}

@incollection{cfchap,
author = {Kimball Milton},
title = {Summation Techniques, {P}ad\'{e} Approximants, and Continued Fractions},
year = {2011},
note = {\url{http://www.nhn.ou.edu/~milton/p5013/chap8.pdf}}
}

@book{cfbook,
title={Continued fractions. Analytic theory and applications},
author={Jones, William B. and Thron, W.J.},
year={1980},
series={Encyclopedia of Mathematics and its Applications},
publisher={Addison-Wesley}
}

@misc{diffusion,
      title={Simple and Effective Masked Diffusion Language Models}, 
      author={Subham Sekhar Sahoo and Marianne Arriola and Yair Schiff and Aaron Gokaslan and Edgar Marroquin and Justin T Chiu and Alexander Rush and Volodymyr Kuleshov},
      year={2024},
     archivePrefix={Advances of Neural Inf. Proc. Systems}
}
\bibliographystyle{abbrv}

\section{Brief Historical Perspective}
\label{app:rel_work}
One of the starting points of artificial neural networks was in the mathematical model of biological neurons known as \textit{artificial neurons} or McColluch-Pitts Neurons proposed in \citep{mcculloch1943logical}. These artificial neurons were remarkably similar to the elements used in modern neural networks, in that their output is a thresholded weighted sum of their inputs. The Multi Layer Perceptron (MLP) \citep{rosenblatt1958perceptron} used multiple layers of neurons with input, hidden and output layers as a simplified model of the nervous system. The Group Method of Data Handling (GMDH) \citep{ivakhnenko1971polynomial} trained a network with an MLP-type structure but each neuron in the network implements a polynomial function of a few input variable, and this was used to train a network that is 8 layers deep.

However, practical learning of networks was made easier after error backpropagation was published \citep{linnainmaa1976taylor} and demonstrated for weight update and learning representation in neural networks \citep{rumelhart1986learning}.

\section{Lemma 2 \citep{cofrnets}}

\label{lem:dCF/da}
We have 
\[
    \frac{\partial}{\partial a_k} \frac{K_{d+1}(a_0, \dots, a_d)}{K_d(a_1, \dots, a_d)} = (-1)^k \left(\frac{K_{d-k}(a_{k+1},\dots,a_d)}{K_d(a_1,\dots,a_d)}\right)^2.
\]

\begin{proof}
To compute the partial derivative of the ratio of continuants above, we first determine the partial derivative of a single continuant $K_k(a_1,\dots,a_k)$ with respect to $a_l$, $l = 1,\dots,k$. We use the representation of $K_k$ as the determinant of the following tridiagonal matrix:
\begin{equation}\label{eqn:contDet}
K_k(a_1, \dots, a_k) = \det 
\begin{bmatrix}
a_1 & 1 \\
-1 & a_2 & \ddots \\
& \ddots & \ddots & 1 \\
& & -1 & a_k
\end{bmatrix}.
\end{equation}
The partial derivatives of a determinant with respect to the matrix entries are given by the \emph{cofactor} matrix:
\[
\frac{\partial\det A}{\partial A_{ij}} = \mathrm{co}(A)_{ij},
\]
where $\mathrm{co}(A)_{ij} = (-1)^{i+j} M_{ij}$ and $M_{ij}$ is the $(i,j)$-minor of $A$. In the present case, with $A$ as the matrix in \eqref{eqn:contDet}, we require partial derivatives with respect to the diagonal entries. Hence 
\[
\frac{\partial K_k(a_1,\dots,a_k)}{\partial a_l} = M_{ll}.
\]
In deleting the $l$th row and column from $A$ to compute $M_{ll}$, we obtain a block-diagonal matrix where the two blocks are tridiagonal and correspond to $a_1,\dots,a_{l-1}$ and $a_{l+1},\dots,a_k$. Applying \eqref{eqn:contDet} to these blocks thus yields 
\begin{equation}\label{eqn:dK/da}
\frac{\partial K_k(a_1,\dots,a_k)}{\partial a_l} = K_{l-1}(a_1,\dots,a_{l-1}) K_{k-l}(a_{l+1},\dots,a_k).
\end{equation}

Returning to the ratio of continuants in the lemma, we use the quotient rule for differentiation and \eqref{eqn:dK/da} to obtain 
\begin{align}
    \frac{\partial}{\partial a_k} \frac{K_{d+1}(a_0, \dots, a_d)}{K_d(a_1, \dots, a_d)} &= \frac{1}{K_d(a_1,\dots,a_d)^2} \left( \frac{\partial K_{d+1}(a_0,\dots,a_d)}{\partial a_k} K_d(a_1,\dots,a_d) \right.\nonumber\\
    &\qquad \qquad {} \left. - K_{d+1}(a_0,\dots,a_d) \frac{\partial K_d(a_1,\dots,a_d)}{\partial a_k} \right)\nonumber\\
    &= \frac{K_{d-k}(a_{k+1},\dots,a_d)}{K_d(a_1,\dots,a_d)^2} \left( K_k(a_0,\dots,a_{k-1}) K_d(a_1,\dots,a_d) \right.\nonumber\\
    &\qquad \qquad {} \left. - K_{d+1}(a_0,\dots,a_d) K_{k-1}(a_1,\dots,a_{k-1}) \right).\label{eqn:dCF/da_1}
\end{align}
We focus on the quantity 
\begin{equation}\label{eqn:dCF/da_2}
    K_k(a_0,\dots,a_{k-1}) K_d(a_1,\dots,a_d) - K_{k-1}(a_1,\dots,a_{k-1}) K_{d+1}(a_0,\dots,a_d)
\end{equation}
in \eqref{eqn:dCF/da_1}. For $k = 0$ (and taking $K_{-1} = 0$), this reduces to $K_d(a_1,\dots,a_d)$. Equation~\eqref{eqn:dCF/da_1} then gives 
\[
\frac{\partial}{\partial a_0} \frac{K_{d+1}(a_0, \dots, a_d)}{K_d(a_1, \dots, a_d)} = \left(\frac{K_{d}(a_1,\dots,a_d)}{K_d(a_1,\dots,a_d)}\right)^2 = 1,
\]
in agreement with the fact that $a_0$ appears only as the leading term in \eqref{eqn:CFcontRatio}. For $k = 1$, \eqref{eqn:dCF/da_2} becomes 
\[
a_0 K_d(a_1,\dots,a_d) - K_{d+1}(a_0,\dots,a_d) = -K_{d-1}(a_2,\dots,a_d)
\]
using \eqref{eqn:recurCont}, and hence 
\[
\frac{\partial}{\partial a_1} \frac{K_{d+1}(a_0, \dots, a_d)}{K_d(a_1, \dots, a_d)} = -\left(\frac{K_{d-1}(a_2,\dots,a_d)}{K_d(a_1,\dots,a_d)}\right)^2.
\]
We generalize from the cases $k = 0$ and $k = 1$ with the following lemma.

\textbf{Lemma 3.} The following identity holds:
\begin{multline*}
K_k(a_0,\dots,a_{k-1}) K_d(a_1,\dots,a_d) - K_{k-1}(a_1,\dots,a_{k-1}) K_{d+1}(a_0,\dots,a_d)\\ = (-1)^k K_{d-k}(a_{k+1},\dots,a_d).
\end{multline*}

Combining \eqref{eqn:dCF/da_1} and Lemma 3 completes the proof.
\end{proof}
\begin{proof}[Proof of Lemma 3]
We prove the lemma by induction. The base cases $k = 0$ and $k = 1$ were shown above and hold moreover for any depth $d$ and any sequence $a_0, \dots, a_d$. Assume then that the lemma is true for some $k$, any $d$, and any $a_0,\dots,a_d$. For $k+1$, we use recursion~\eqref{eqn:recurCont} to obtain
\begin{align*}
    &K_{k+1}(a_0,\dots,a_{k}) K_d(a_1,\dots,a_d) - K_{k}(a_1,\dots,a_{k}) K_{d+1}(a_0,\dots,a_d)\\
    &\quad = \bigl(a_0 K_k(a_1,\dots,a_k) + K_{k-1}(a_2,\dots,a_k)\bigr) K_d(a_1,\dots,a_d)\\
    &\quad \qquad {} - K_{k}(a_1,\dots,a_{k}) \bigl(a_0 K_d(a_1,\dots,a_d) + K_{d-1}(a_2,\dots,a_d)\bigr)\\
    &\quad = K_{k-1}(a_2,\dots,a_k) K_d(a_1,\dots,a_d) - K_{k}(a_1,\dots,a_{k}) K_{d-1}(a_2,\dots,a_d).
\end{align*}
We then recognize the last line as an instance of the identity for $k$, depth $d-1$, and sequence $a_1,\dots,a_d$. Applying the inductive assumption, 
\begin{align*}
    &K_{k+1}(a_0,\dots,a_{k}) K_d(a_1,\dots,a_d) - K_{k}(a_1,\dots,a_{k}) K_{d+1}(a_0,\dots,a_d)\\
    &\quad = -(-1)^k K_{d-1-k}(a_{k+2},\dots,a_d)\\
    &\quad = (-1)^{k+1} K_{d-(k+1)}(a_{(k+1)+1},\dots,a_d),
\end{align*}
as required.
\end{proof}

\section{Example Generations}

In Figures \ref{fig:gpt2_owt}, \ref{fig:gpt2_owt_f}, \ref{fig:gpt2_owt_a} and \ref{fig:gpt2_owt_af} we see example generations of GPT2-xl, CoFrGeNet-F, CoFrGeNet-A and CoFrGeNet respectively when pre-trained on OWT dataset. While in Figures \ref{fig:gpt2_gw}, \ref{fig:gpt2_gw_f}, \ref{fig:gpt2_gw_a} and \ref{fig:gpt2_gw_af} we see example generations of GPT2-xl, CoFrGeNet-F, CoFrGeNet-A and CoFrGeNet respectively when pre-trained on GW dataset.

\begin{figure}
\centering
\includegraphics[width=\textwidth]{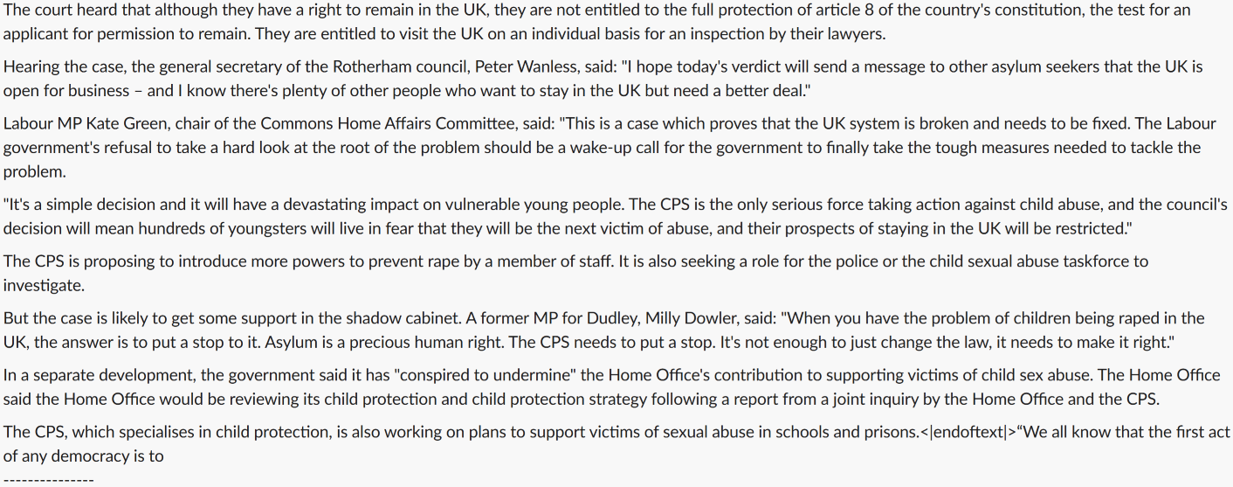}
  \caption{GPT2-xl example generation when pre-trained on OWT.}
\label{fig:gpt2_owt}
\end{figure}

\begin{figure}
\centering
\includegraphics[width=\textwidth]{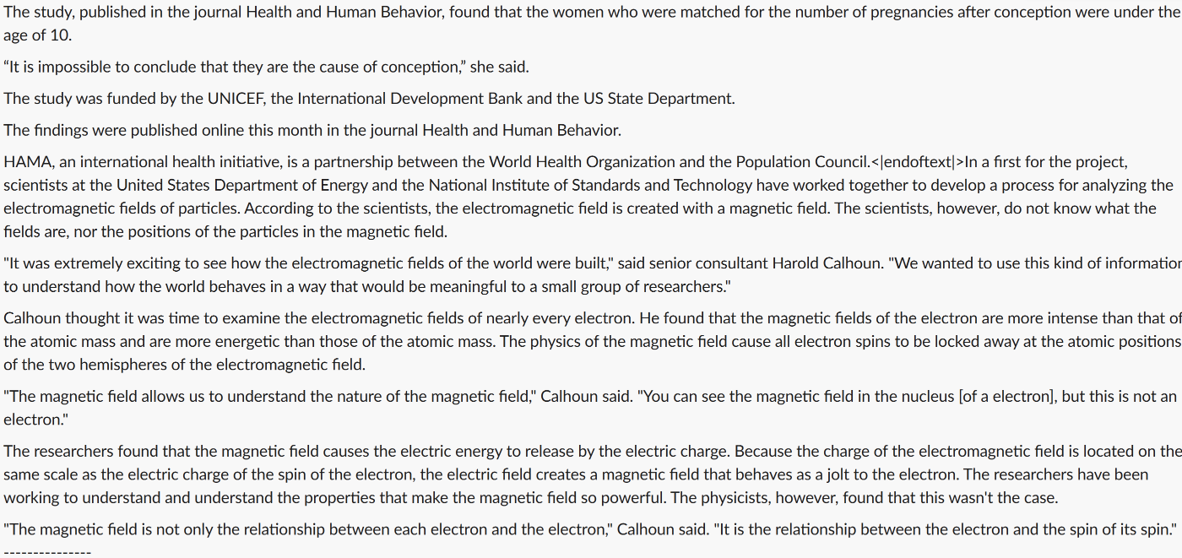}
  \caption{CoFrGeNet-F example generation when pre-trained on OWT.}
\label{fig:gpt2_owt_f}
\end{figure}

\begin{figure}
\centering
\includegraphics[width=\textwidth]{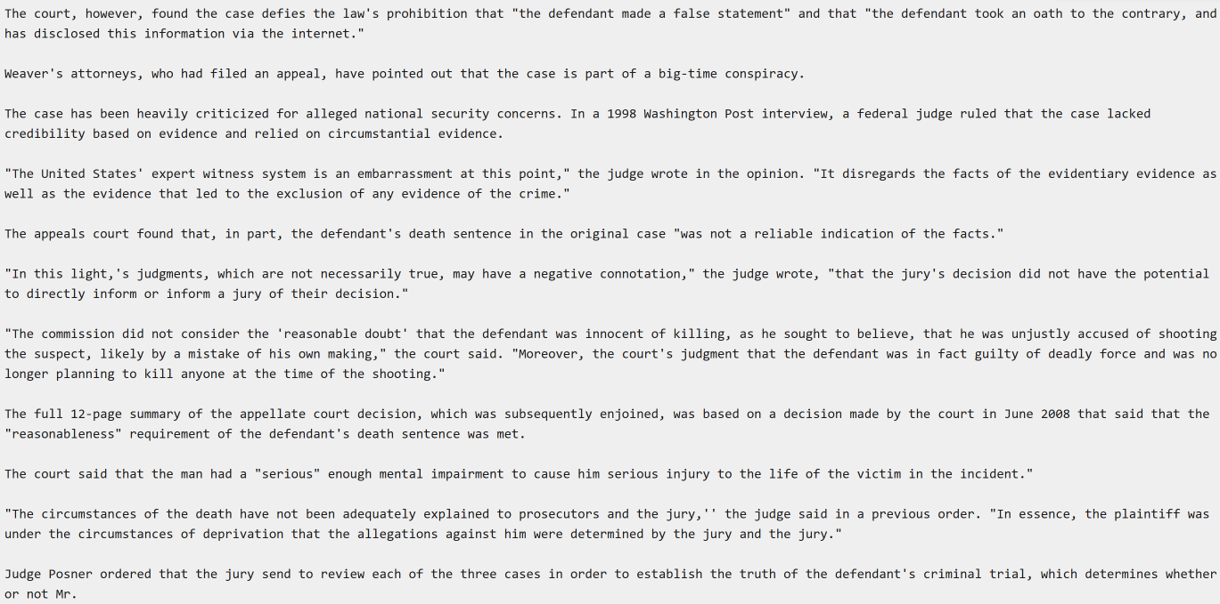}
  \caption{CoFrGeNet-A example generation when pre-trained on OWT.}
\label{fig:gpt2_owt_a}
\end{figure}

\begin{figure}[!h]
\centering
\includegraphics[width=\textwidth]{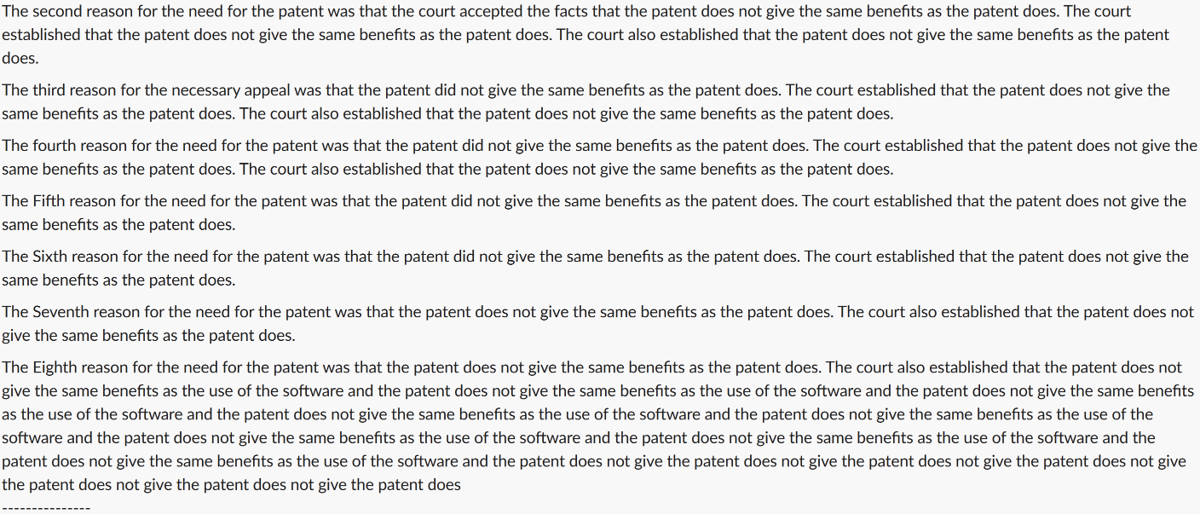}
  \caption{CoFrGeNet example generation when pre-trained on OWT.}
\label{fig:gpt2_owt_af}
\end{figure}

\begin{figure}
\centering
\includegraphics[width=\textwidth]{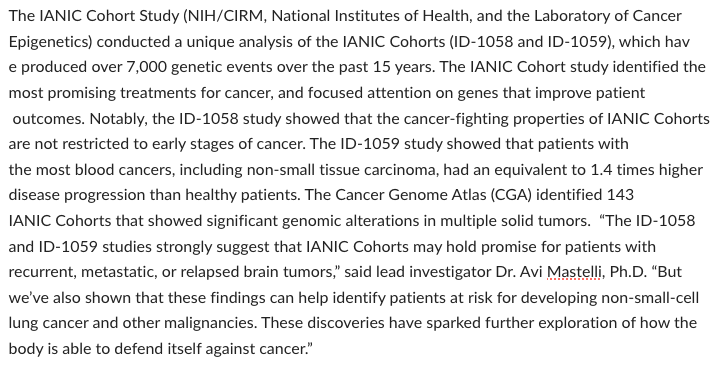}
  \caption{GPT2-xl example generation when pre-trained on GneissWeb.}
\label{fig:gpt2_gw}
\end{figure}

\begin{figure}
\centering
\includegraphics[width=\textwidth]{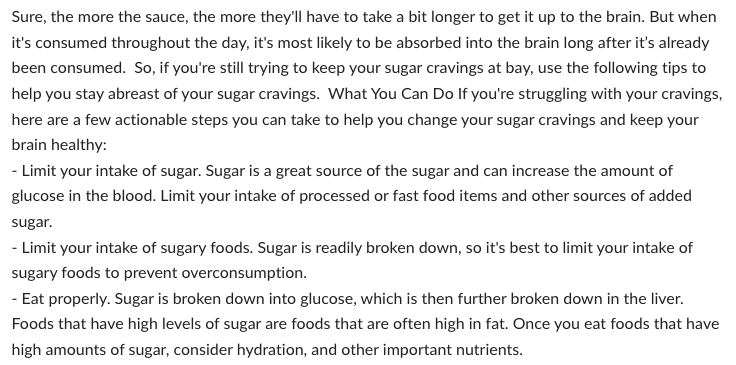}
  \caption{CoFrGeNet-F example generation when pre-trained on GneissWeb.}
\label{fig:gpt2_gw_f}
\end{figure}

\begin{figure}
\centering
\includegraphics[width=\textwidth]{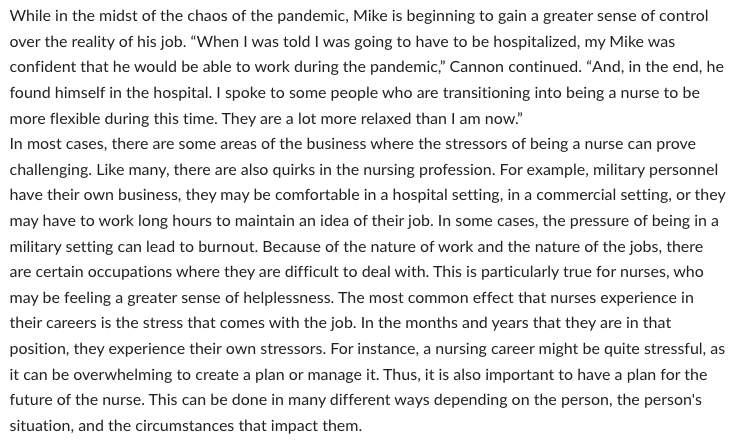}
  \caption{CoFrGeNet-A example generation when pre-trained on GneissWeb.}
\label{fig:gpt2_gw_a}
\end{figure}

\begin{figure}
\centering
\includegraphics[width=\textwidth]{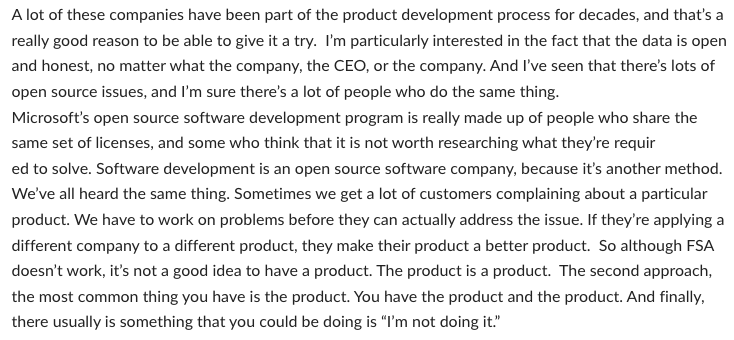}
  \caption{CoFrGeNet example generation when pre-trained on GneissWeb.}
\label{fig:gpt2_gw_af}
\end{figure}

\begin{table}[htbp]
\scriptsize
\centering
\caption{Downstream task accuracies on GLUE benchmark after finetuning the pre-trained models. The first column is the pre-training dataset. Results are mean$\pm$std with the best means bolded.}
\begin{tabular}{|p{4mm}|p{21mm}|p{11mm}|p{11mm}|p{11mm}|p{11mm}|p{11mm}|p{11mm}|p{11mm}|p{11mm}|}
  \hline
    \textbf{Data} &\textbf{Model} & \textbf{MNLI} & \textbf{QQP}& \textbf{QNLI} & \textbf{SST2} & \textbf{COLA} & \textbf{MRPC} & \textbf{RTE} & \textbf{WNLI}\\
 \hline\hline
 \multirow{6}{*}{OWT} & GPT2-xl \tiny{(1.5B)}& $86.89$\tiny{$\pm.15$}&$88.93$\tiny{$\pm.67$} &$91.35$\tiny{$\pm.34$}&$93.56$\tiny{$\pm.24$}&$81.78$\tiny{$\pm.38$}&$79.83$\tiny{$\pm.26$}&$60.27$\tiny{$\pm.22$}&$58.28$\tiny{$\pm.28$}\\
 &CoFrGeNet-F \tiny{(985M)}&$\bf 87.26$\tiny{$\pm.18$}&$\bf 89.95$\tiny{$\pm.12$} &$\bf 91.89$\tiny{$\pm.34$}&$\bf 94.16$\tiny{$\pm.29$}&$\bf 82.59$\tiny{$\pm.23$}&$\bf 80.21$\tiny{$\pm.19$}&$\bf 61.35$\tiny{$\pm.32$}&$\bf 58.30$\tiny{$\pm.16$}\\
 &CoFrGeNet-A \tiny{(1.21B)}&$86.94$\tiny{$\pm.12$}&$89.31$\tiny{$\pm.42$} &$91.74$\tiny{$\pm.31$}&$93.83$\tiny{$\pm.72$}&$81.77$\tiny{$\pm.25$}&$79.89$\tiny{$\pm.14$}&$60.91$\tiny{$\pm.92$}&$58.28$\tiny{$\pm.17$}\\
 &CoFrGeNet \tiny{(798M)}&$87.11$\tiny{$\pm.09$}&$89.36$\tiny{$\pm.23$} &$91.79$\tiny{$\pm.25$}&$93.91$\tiny{$\pm.15$}&$81.97$\tiny{$\pm.14$}&$79.93$\tiny{$\pm.17$}&$61.25$\tiny{$\pm.46$}&$58.29$\tiny{$\pm.19$}\\
&Synthesizer-D \tiny{(1.2B)}&$84.93$\tiny{$\pm.34$}&$86.82$\tiny{$\pm.34$} &$90.13$\tiny{$\pm.51$}&$91.34$\tiny{$\pm.54$}&$80.15$\tiny{$\pm.72$}&$77.95$\tiny{$\pm.25$}&$59.83$\tiny{$\pm.35$}&$58.28$\tiny{$\pm.92$}\\
&Sparse Attn \tiny{(1.21B)}&$85.27$\tiny{$\pm.63$}&$86.38$\tiny{$\pm.33$} &$90.93$\tiny{$\pm.18$}&$92.72$\tiny{$\pm.21$}&$80.76$\tiny{$\pm.28$}&$77.42$\tiny{$\pm.41$}&$59.36$\tiny{$\pm.29$}&$58.27$\tiny{$\pm.25$}\\
 \hline
 \multirow{6}{*}{GW} & GPT2-xl \tiny{(1.5B)}& $78.28$\tiny{$\pm.82$}&$86.83$\tiny{$\pm.17$}&$\bf 82.93$\tiny{$\pm.37$}&$91.82$\tiny{$\pm.22$}&$74.18$\tiny{$\pm.82$}&$77.72$\tiny{$\pm.93$}&$60.19$\tiny{$\pm.01$}&$\bf 58.33$\tiny{$\pm.07$}\\
 &CoFrGeNet-F \tiny{(985M)}&$\bf 79.62$\tiny{$\pm.63$}&$\bf 87.26$\tiny{$\pm.25$}&$82.73$\tiny{$\pm.53$}&$\bf 92.36$\tiny{$\pm.45$}&$\bf 74.83$\tiny{$\pm.56$}&$\bf 78.01$\tiny{$\pm.34$}&$\bf 61.35$\tiny{$\pm.08$}&$\bf 58.33$\tiny{$\pm.04$}\\
 &CoFrGeNet-A \tiny{(1.21B)}&$78.42$\tiny{$\pm.34$}&$86.17$\tiny{$\pm.46$}&$82.51$\tiny{$\pm.36$}&$91.86$\tiny{$\pm.36$}&$74.15$\tiny{$\pm.43$}&$77.37$\tiny{$\pm.83$}&$60.85$\tiny{$\pm.06$}&$\bf 58.33$\tiny{$\pm.06$}\\
 &CoFrGeNet \tiny{(798M)}&$79.05$\tiny{$\pm.37$}&$86.98$\tiny{$\pm.22$}&$82.12$\tiny{$\pm.28$}&$92.13$\tiny{$\pm.73$}&$74.38$\tiny{$\pm.74$}&$77.95$\tiny{$\pm.73$}&$61.11$\tiny{$\pm.04$}&$\bf 58.33$\tiny{$\pm.02$}\\
 &Synthesizer-D \tiny{(1.2B)}&$77.56$\tiny{$\pm.12$}&$86.35$\tiny{$\pm.61$} &$80.38$\tiny{$\pm.83$}&$91.25$\tiny{$\pm.71$}&$73.27$\tiny{$\pm.73$}&$76.73$\tiny{$\pm.27$}&$59.26$\tiny{$\pm.22$}&$58.24$\tiny{$\pm.97$}\\
&Sparse Attn \tiny{(1.21B)}&$77.67$\tiny{$\pm.38$}&$86.41$\tiny{$\pm.82$} &$80.77$\tiny{$\pm.16$}&$91.16$\tiny{$\pm.16$}&$72.83$\tiny{$\pm.26$}&$76.62$\tiny{$\pm.81$}&$59.39$\tiny{$\pm.38$}&$58.28$\tiny{$\pm.28$}\\
 \hline
  \end{tabular}
 \label{tab:glue_sd}
\end{table}

\end{document}